\documentclass[twoside]{article}

\usepackage[utf8]{inputenc} % allow utf-8 input
\usepackage[T1]{fontenc}    % use 8-bit T1 fonts
\usepackage{lmodern}  % Uses Type 1 scalable fonts instead of Type 3
\usepackage{hyperref}       % hyperlinks
\usepackage{url}            % simple URL typesetting
\usepackage{booktabs}       % professional-quality tables
\usepackage{amsfonts}       % blackboard math symbols
\usepackage{nicefrac}       % compact symbols for 1/2, etc.
\usepackage{microtype}      % microtypography
\usepackage{xcolor}         % colors
\usepackage{wrapfig}
\usepackage{graphicx}

\usepackage{subcaption}
\usepackage{caption}

\usepackage{amsmath}
\usepackage{amssymb}
\usepackage{mathtools}
\usepackage{amsthm}
\usepackage{enumitem}
\usepackage{xspace}

\usepackage[capitalize,noabbrev]{cleveref}

\usepackage[textsize=tiny]{todonotes}

\usepackage{ninecolors}

\usepackage{thm-restate}

\usepackage[accepted]{aistats2025}
\usepackage[round]{natbib}

\usepackage{AISTATS2025/algorithm}
\usepackage{AISTATS2025/algorithmic}

\usepackage{Definitions}
\usepackage{AISTATS2025/antoinemacros}
\newcommand{\algabb}{SpecIV\xspace}
\newcommand{\algabbpcl}{SpecPCL\xspace}
\renewcommand{\epsilon}{\varepsilon}

\begin{document}


\runningauthor{Haotian Sun$^*$, Antoine Moulin$^*$, Tongzheng Ren$^*$, Arthur Gretton, Bo Dai}

\twocolumn[

\aistatstitle{Spectral Representation for Causal Estimation with Hidden Confounders}

\aistatsauthor{Haotian Sun$^*$ \And Antoine Moulin$^{* \dagger}$ \And Tongzheng Ren$^{* \dagger}$}
\aistatsaddress{
Georgia Tech \\ \texttt{haotian.sun@gatech.edu}
\And
Universitat Pompeu Fabra \\ \texttt{antoine.moulin@upf.edu}
\And
UT Austin \\ \texttt{rtz19970824@gmail.com}
}
\vspace{-10pt}
\aistatsauthor{Arthur Gretton \And Bo Dai}

\aistatsaddress{University College London \& Google DeepMind \\ \texttt{arthur.gretton@gmail.com} \And Georgia Tech \& Google DeepMind\\ \texttt{bodai@cc.gatech.edu}}
\vspace{-15pt}
\aistatsaddress{\normalfont $^*$ {Equal contribution with random order} \qquad $^\dagger$ {Part of the work was done while at Google DeepMind}}

]

\begin{abstract}
  We study the problem of causal effect estimation in the presence of unobserved confounders, focusing on two settings: instrumental variable (IV) regression with additional observed confounders, and proxy causal learning. Our approach uses a singular value decomposition of a conditional expectation operator combined with a saddle-point optimization method. In the IV regression setting, this can be viewed as a neural network generalization of the seminal approach due to \cite{darolles2011nonparametric}. Saddle-point formulations have recently gained attention because they mitigate the double-sampling bias and are compatible with modern function approximation methods. We provide experimental validation across various settings and show that our approach outperforms existing methods on common benchmarks.
\end{abstract}
\vspace{-4mm}
\section{INTRODUCTION}
\vspace{-2mm}

Estimating causal effects is a fundamental problem across diverse fields, including economics, epidemiology, and social sciences. Unobserved confounders---variables influencing both the treatment and the outcome---present a major challenge to traditional estimation methods as they introduce spurious correlations leading to biased and inconsistent estimates (see Equation 6.1 and Chapter 12 in \citealp{stock2007econometrics}). One approach to address unobserved confounding involves using additional variables to help identify the causal effect. Instrumental variable (IV) regression \citep{wright1928tariff, stock2003retrospectives} and proxy causal learning (PCL) \cite{kuroki2014measurement} are two prominent examples of this approach.

IV regression aims to solve an ill-posed inverse problem of the form $E f = r$, where $f$ lives in the same space as the causal effect to be identified, $r$ is the expected output conditioned on an \emph{instrument}, and $E$ is the corresponding conditional expectation operator. An instrument is a variable that provides exogenous variation in the treatment---variation that is uncorrelated with the unobserved confounder---allowing us to isolate the treatment's causal effect on the outcome. The problem is ill-posed because solving for $f$ (assuming a solution exists) typically involves the inverse eigenvalues of $E$, which can be arbitrarily close to zero. A common approach to account for this is to assume the target function $f$ satisfies a \emph{source condition} \citep{engl1996regularization}, which relates its smoothness to the operator $E$ and ensures that the target primarily depends on the largest eigenvalues of $E$. Existing methods for efficiently solving this inverse problem can be broadly categorized into \emph{two-stage estimation methods} \citep{newey2003instrumental, darolles2011nonparametric, chen2018sieve, singh2019kernel}, and \emph{conditional moment methods} \citep{dai2017learning, dikkala2020minimax, liao2020provably, bennett2019deep, bennett2023minimax, bennett2023source}.

Two-stage methods aim to minimize an expected mean-squared error (MSE) of the form $\norm{E f - r}^2$ and are designed to deal with the conditional expectation nested inside the square function. Stage 1 performs a regression to estimate the conditional means, and Stage 2 regresses the outcome on the estimates obtained in Stage 1. Prior works differ in the parametrization of Stage 1. \cite{hartford2017deepiv} introduce a deep mixture model to estimate the conditional means, while \cite{darolles2011nonparametric} estimate the conditional densities with kernel density estimators. \cite{singh2019kernel} instead learn a conditional mean embedding \citep{song2009cme,GruLevBalPatetal12,li2022optimal} of features that map the input to a reproducing kernel Hilbert space (RKHS), providing a nonlinear generalization of two-stage least-squares. However, this approach uses fixed and pre-defined feature dictionaries.

Alternatively, features can be learned adaptively within the two-stage least-squares (2SLS) framework \citep{xu2020dfiv},
using gradient-based methods to minimize the same MSE. However, a naive approach leads to biased estimates due to the presence of the conditional expectation, failing to minimize the loss of interest. This \emph{double sampling bias} has been a particular challenge in causal inference and reinforcement learning (see Chapter 11.5 in \citealp{sutton2018reinforcement}, \citealp{antos2008learning, bradtke1996linear}). While samples from the conditional distribution would bypass the issue, we rarely have enough data to obtain multiple samples for each specific conditioning value. This motivates conditional moment methods which avoid minimizing the squared error directly. Instead, these methods consider a saddle-point optimization problem of the form $\min_{f \in \calF} \max_{g \in \calG} \calL \sprt{f, g}$ for some function classes $\calF$ and $\calG$. This formulation can be derived for instance via the Lagrangian of a feasibility problem \citep{bennett2023minimax} or the Fenchel conjugate of the square function together with an interchangeability result \citep{dai2017learning, dikkala2020minimax, liao2020provably}.

These saddle-point formulations involve only expectations with respect to the \emph{joint} data distribution. While facilitating first-order methods and function approximation, they require strong assumptions on the function classes used \citep{dikkala2020minimax, liao2020provably, bennett2023inference, bennett2023source}. A natural one is \emph{realizability}; it requires the function classes to contain the objects to be estimated, \eg, $f_0 \in \calF$ where $f_0$ is a solution to the problem. Removing it at the cost of a constant error due to misspecification is possible. However, due to the double sampling issue, most algorithms cannot work under realizability alone, and an additional expressivity assumption is required. A common one is \emph{closedness} and requires stability of the function classes under the conditional expectation operator $E$ and/or its adjoint $E^\star$, \eg, $E \calF \subset \calG$ and/or vice-versa. While \cite{bennett2023minimax} manage to avoid the closedness assumption, they use a specific source condition and require an additional realizability assumption on the dual function class $\calG$. It is unclear whether their method can adapt to different degrees of smoothness. Let $\hat{f}_n$ be an estimator of $f_0$, and $n$ the number of samples at hand. \cite{dikkala2020minimax} show, under these assumptions, that the projected MSE, $\norm{E (\hat{f}_n - f_0)}_2$, converges at a rate of $\calO \sprt{n^{-1/2}}$. \cite{liao2020provably} achieve a stronger $L_2$ guarantee, $\norm{\hat{f}_n - f_0}_2 = \calO \sprt{n^{-1/6}}$, with an additional source condition and uniqueness of the solution. Their analysis crucially hinges on the latter to convert a projected MSE guarantee to an $L_2$ guarantee by controlling the \emph{measure of ill-posedness} $\sup_{f \in \calF} \frac{\norm{f - f_0}_2^2}{\norm{E \sprt{f - f_0}}_2^2}$, which can be unbounded if several solutions exist. On the other hand, \cite{bennett2023minimax} obtain $\norm{\hat{f}_n - f_0}_2 = \calO \sprt{n^{-1/4}}$ with only a source condition and realizable classes, and \cite{bennett2023source} obtain guarantees on both the projected MSE and the $L_2$ norm with two additional closedness assumptions.

The assumptions required for IV regression can be restrictive. Prior work extends IV regression to accommodate observable confounders \cite{horowitz2011applied, xu2020dfiv}. When a valid instrument is unavailable, an alternative is to use \emph{proxy variables}, which contain relevant information on the unobserved confounder. \cite{kuroki2014measurement} provide necessary conditions on the proxy variables for identifying the true causal effect; \cite{miao2018identifying} generalize these conditions. Recent methods estimate causal effects in the proxy causal learning (PCL) setting, including \citep{deaner2018proxy,mastouri2021proximal} for fixed feature dictionaries, and \citep{xu2021dfpv,kompa2022deep} for adaptive feature dictionaries. However, the theoretical understanding of IV regression with observed confounders and PCL is less mature than that of IV regression.

While conditional moment methods offer strong statistical guarantees for IV regression, choosing the function classes $\calF$ and $\calG$ remains unclear. We propose a method based on a low-rank assumption on the conditional operator $E$, drawing inspiration from reinforcement learning \citep{jin2020provably}. This assumption yields function classes satisfying the crucial realizability and closedness conditions, enabling optimization over finite-dimensional variables. We leverage this to derive efficient algorithms for IV regression with observed confounders (IV-OC) and PCL. Section~\ref{sec:prelim} formalizes the settings and derives the saddle-point formulations. We introduce the low-rank assumption and show how to leverage it in Section~\ref{sec:method}. Importantly, unlike previous works focusing on standard IV regression, our method also learns an adaptive basis for IV with observed confounder and PCL. One component of our method is a representation learning algorithm inspired from \cite{wang2022spectral}, discussed in Section~\ref{sec:spec_causal}. Finally, we provide experimental validation in Section~\ref{sec:experiments} and show our approach outperforms existing methods on IV and PCL benchmarks.

%%%%%%%%%%%%%%%%%%%%%%%%%%%%%%%%%%%%%%%%%%%%%%%%%%%%%%%%%%%%
\vspace{-2mm}
\section{PRELIMINARIES} \label{sec:prelim}
\vspace{-2mm}
%%%%%%%%%%%%%%%%%%%%%%%%%%%%%%%%%%%%%%%%%%%%%%%%%%%%%%%%%%%%

We formalize the three settings of interest and derive their equivalent saddle-point formulations.

\textbf{Notation}. For a random variable $A$ taking values in $\calA$, let $\bbP_A$ denote its probability distribution, and $L_2 \sprt{\bbP_A}$ denote the space of $\bbP_A$-square-integrable functions $f: \calA \rightarrow \bbR$. Given another random variable $B$, we denote $\bbP_{A \given B}$ a regular version of the conditional distribution of $A$ given $B$. Given $n$ samples, $\bbE_n$ denotes the empirical expectation. denote the Euclidean inner product as $\inner{\cdot}{\cdot}$, and the range of an operator $E$ as $\calR \sprt{E}$.

\begin{figure*}[!h]
\vspace{-2mm}
\begin{center}
\begin{subfigure}{.33\linewidth}
  \centering
  \begin{tikzpicture}[every node/.style={circle,thick,draw}]
    \node[fill=lightgray] (Z) at (0.1,0) {Z};
    \node[fill=lightgray] (X) at (1.9,0) {X};
    \node[fill=lightgray] (Y) at (3.8,0) {Y};
    \node[minimum size=20pt] (e) at (2.85,1.73) {$\epsilon$};
        \draw [-stealth, thick](Z) -- (X);
        \draw [-stealth, thick](X) -- (Y);
        \draw [-stealth, thick](e) -- (Y);
        \draw [-stealth, thick](e) -- (X);
  \end{tikzpicture}
  \caption{}
  \label{fig:IV_Graph}
\end{subfigure}%
\begin{subfigure}{.33\linewidth}
  \centering
    \begin{tikzpicture}[every node/.style={circle,thick,draw}]
    \node[fill=lightgray] (Z) at (0.2,0) {Z};
    \node[fill=lightgray] (X) at (2,0) {X};
    \node[fill=lightgray] (Y) at (4,0) {Y};
    \node[fill=lightgray] (O) at (2,1.73) {O};
    \node[minimum size=20pt] (e) at (4,1.73) {$\epsilon$};
    \draw [-stealth, thick](Z) -- (X);
    \draw [-stealth, thick](X) -- (Y);
    \draw [-stealth, thick](e) -- (Y);
    \draw [-stealth, thick](e) -- (X);
    \draw [-stealth, thick](O) -- (X);
    \draw [-stealth, thick](O) -- (Y);
    \end{tikzpicture}  
    \caption{}
    \label{fig:IV_Graph_obs_confounder}
\end{subfigure}
\begin{subfigure}{.325\textwidth}
  \centering
   \begin{tikzpicture}[every node/.style={circle,thick,draw}]
    \node[fill=lightgray] (Z) at (1.5,1.23) {Z};
    \node[fill=lightgray] (W) at (4.5,1.23) {W};
    \node[fill=lightgray] (X) at (2,0) {X};
    \node[fill=lightgray] (Y) at (4,0) {Y};
    \node[minimum size=20pt] (e) at (3,1.73) {$\epsilon$};
        \draw [stealth-stealth, dashed, thick](Z) -- (X);
        \draw [-stealth, thick](X) -- (Y);
        \draw [-stealth, thick](W) -- (Y);
        \draw [-stealth, thick](e) -- (Y);
        \draw [-stealth, thick](e) -- (X);
        \draw [-stealth, thick](e) -- (Z);
        \draw [stealth-stealth, dashed, thick](e) -- (W);
    \end{tikzpicture}  
    \caption{}
  \label{fig:PCL_Graph}
\end{subfigure}
\caption{Causal graphs for the three settings: (a) IV regression, (b) IV regression with observed confounder ($O$), and (c) proxy causal inference. Gray nodes represent observed variables; white nodes are unobserved.}
\label{fig:causal_graph}
\vspace{-3mm}
\end{center} 
\end{figure*}

%%%%%%%%------------------------------------------------
\vspace{-2mm}
\subsection{Causal Estimation with Hidden Confounders} \label{sub:causal_est}
\vspace{-1mm}
%%%%%%%%------------------------------------------------

%IV regression and PCL are statistical techniques for estimating causal effects in the presence of unobserved confounding. Solving these problems can be regarded as finding a solution to a Fredholm equation of the first kind for different conditional expectation operators.

Given random variables $X$, $S$ and $Y$, in IV regression with and without observed confounders, we aim to find a function $f$ satisfying
\begin{equation} \label{eq:fredholm_eq}
    \bbE \sbrk{Y - f \sprt{X, S} \given S} = 0\,.
\end{equation}
Here, $Y$ is the outcome, $X$ is the treatment, and $S$ contains the accessible side information. IV regression and PCL differ in their assumptions and applicability.

IV regression is used when the unobserved confounder affects the outcome linearly. As shown in Figure~\ref{fig:IV_Graph}, it involves an \emph{instrument} $Z$ that (i) is independent of the outcome $Y$ conditional on the treatment $X$ and the confounder $\epsilon$, $Z \indep Y \given \sprt{X, \epsilon}$, and (ii) satisfies $\bbE \sbrk{\epsilon \given Z} = 0$. In this setting, the side information is given by the instrument, $S = Z$ in Equation~\eqref{eq:fredholm_eq}. Since the instrument $Z$ does not affect the outcome $Y$, we only consider functions of the treatment $X$ alone
\begin{equation} \label{eq:fredholm_iv}
    \bbE \sbrk{Y - f \sprt{X} \given Z} = 0\,.
\end{equation}
With an additional observed confounder $O$, the available information becomes $S = \sprt{Z, O}$ (Figure~\ref{fig:IV_Graph_obs_confounder}), and we are interested in solving the equation
\begin{equation} \label{eq:fredholm_ivo}
    \bbE \sbrk{Y - f \sprt{X, O} \given Z, O} = 0\,,
\end{equation}
where the input space of $f$ now includes the observed confounder because of its effect on the outcome.

PCL~\citep{miao2018identifying,deaner2018proxy} uses two \emph{proxies} correlated with the unobserved confounder $\epsilon$. One proxy, $Z$, correlates with the treatment $X$; the other, $W$, correlates with the outcome $Y$ (Figure~\ref{fig:PCL_Graph}). The proxies $Z$ and $W$ satisfy the independence properties $Z \indep Y \given \sprt{X, \epsilon}$ and $W \indep \sprt{Z, X} \given \epsilon$. Here, $S = \sprt{Z, W}$, and we consider
\begin{equation} \label{eq:fredholm_pcl}
    \bbE \sbrk{Y - f \sprt{X, W} \given X, Z} = 0\,.
\end{equation}
These equations are ill-posed inverse problems we are interested in. Next, we derive their equivalent saddle-point formulations.

%%%%%%%%------------------------------------------------
\vspace{-2mm}
\subsection{Primal-Dual Framework for Causal Estimation} \label{sub:primal_dual}
\vspace{-1mm}
%%%%%%%%------------------------------------------------

We now derive a saddle-point formulation from Equation~\eqref{eq:fredholm_eq}. Similar derivations apply to PCL. Prior work minimizes the mean-squared error plus a regularizer, $\Omega$, to address ill-posedness
\begin{equation} \label{eq:fredholm-regression}
    \min_{f \in \calF} \bbE \sbrk{\bbE \sbrk{Y - f \sprt{X, S} \given S}^2} + \lambda \Omega \sprt{f} \triangleq \calE \sprt{f}\,,
\end{equation}
where $\calF \subset L_2 \sprt{\bbP_{X S}}$ is convex, $\lambda > 0$, and $\Omega$ is a strongly convex regularizer (\eg, a $L_2$ regularizer~\cite{darolles2011nonparametric} or an RKHS norm~\cite{singh2019kernel, ZhangImaizumiScholkopfMuandet2023,wang2022spectral}). Following, \eg, \cite{dai2017learning}, we start with the Fenchel conjugate of the square function to write
\begin{align*}
    &\bbE \sbrk{\bbE \sbrk{Y - f \sprt{X, S} \given S}^2} \\
    \;\;&= 2\, \bbE \sbrk{\max_{g \in \bbR} \scbrk{g \, \bbE \sbrk{Y - f \sprt{X, S} \given S} - \frac12 g^2}} \\
    &= 2 \, \max_{g \in L_2 \sprt{\bbP_S}} \bbE \sbrk{g \sprt{S} \bbE \sbrk{Y - f \sprt{X, S} \given S} - \frac12 g \sprt{S}^2} \\
    &= 2 \, \max_{g \in L_2 \sprt{\bbP_S}} \bbE \sbrk{g \sprt{S} \sprt{Y - f \sprt{X, S}} - \frac12 g \sprt{S}^2}\,,
\end{align*}
where the second equality follows from \cite[Lemma~1]{dai2017learning}, and the last equality uses the tower rule. Thus, given a convex class $\calG \subset L_2 \sprt{\bbP_S}$, we consider
\begin{equation} \label{eq:fredholm-minmax}
    \min_{f \in \calF} \max_{g \in \calG} \underbrace{\bbE \sbrk{g \sprt{S} \sprt{Y - f \sprt{X, S}} - \frac12 g \sprt{S}^2} + \lambda \Omega \sprt{f}}_{\triangleq \calL \sprt{f, g}}\,.
\end{equation}
$\calL$ is strongly-convex in $f$ and strongly-concave in $g$. Unlike $\calE$, the absence of conditional expectations makes it straightforward to derive unbiased estimators of $\calL$ or its derivatives and avoid the \emph{double sampling bias} mentioned earlier. Denoting $\calL_n$ as the empirical counterpart of $\calL$, for any functions $f, g$, we have $\bbE \sbrk{\calL_n \sprt{f, g}} = \calL \sprt{f, g}$.

For IV regression, we consider $S = Z$ and $\calF \subset L_2 \sprt{\bbP_X}$ in Problem~\eqref{eq:fredholm-minmax}, or $S = \sprt{Z, O}$ and $\calF \subset L_2 \sprt{\bbP_{XO}}$ with an observable confounder. For PCL, the min-max problem becomes
\begin{align} \label{eq:pcl-minmax}
    \min_{f \in \calF} \max_{g \in \calG}\ &\bbE \sbrk{g \sprt{ X, Z} \cdot \sprt{Y - f \sprt{X, W}} - \frac12 g \sprt{X, Z}^2} \nonumber \\
    &\quad+ \lambda \Omega \sprt{f}\,.
\end{align}
For consistency of the estimator
\begin{equation*}
    \wh{f}_n = \argmin_{f \in \calF} \max_{g \in \calG} \calL_n \sprt{f, g}\,,
\end{equation*}
the classes $\calF$ and $\calG$ must ensure that the minimizer $f_0 = \argmin_{f \in \calF} \max_{g \in \calG} \calL \sprt{f, g}$ solves the original inverse problem. A key challenge is choosing the function classes $\calF$ and $\calG$ to be small enough for optimization to be tractable, yet large enough for realizability to hold.

%%%%%%%%%%%%%%%%%%%%%%%%%%%%%%%%%%%%%%%%%%%%%%%%%%%%%%%%%%%%
\section{CHARACTERIZING THE FUNCTION CLASSES}
\label{sec:method}
%%%%%%%%%%%%%%%%%%%%%%%%%%%%%%%%%%%%%%%%%%%%%%%%%%%%%%%%%%%%

We first characterize the function classes F and G for IV regression under a low-rank assumption, then generalize this to the IV-OC and PCL settings. For clarity, this section assumes the feature maps introduced in the following assumptions are known; Section~\ref{sec:spec_causal} discusses learning these mappings.

%%%%%%%%------------------------------------------------
\subsection{Instrument Variable Regression}
%%%%%%%%------------------------------------------------

Let $E: L_2 \sprt{\bbP_X} \to L_2 \sprt{\bbP_Z}$ be the conditional expectation operator defined as $E f = \bbE \sbrk{f \sprt{X} \given Z}$ for any $f \in L_2 \sprt{\bbP_X}$, and assume Equation~\eqref{eq:fredholm_iv} has a solution.
\begin{assumption} \label{asp:existence-sol-iv}
    There exists $f_0 \in L_2 \sprt{\bbP_X}$ such that $E f_0 = \bbE \sbrk{Y \given Z}$.
\end{assumption}

Our goal is to find subspaces of $L_2 \sprt{\bbP_X}$ and $L_2 \sprt{\bbP_Z}$ such that the saddle-point of Problem~\eqref{eq:fredholm-minmax} yields a solution to Equation~\eqref{eq:fredholm_iv}. We now introduce the key low-rank assumption on the conditional and marginal distributions of $X$ and $Z$.
\begin{assumption} \label{asp:low-rank-iv}
    The distributions $\bbP_X$ and $\bbP_{X \given Z = z}$, for any $z \in \calZ$, have densities with respect to the Lebesgue measure, denoted $p_X$ and $p_{X \given Z} \sprt{\cdot \given z}$, respectively. Furthermore, there exist feature maps $\phi: \calX \rightarrow \bbR^d$ and $\psi: \calZ \rightarrow \bbR^d$ such that for any $x, z$,
    \begin{equation} \label{eq:iv_factorization}
        p_{X \given Z} \sprt{x \given z} = p_X \sprt{x} \inp{\phi \sprt{x}, \psi \sprt{z}}\,.
    \end{equation}
\end{assumption}
We denote $\Psi: \bbR^d \to L_2 \sprt{\bbP_Z}$ and $\Phi: \bbR^d \to L_2 \sprt{\bbP_X}$ the operators defined for any vector $u$ as $\Psi u = \inp{\psi \sprt{\cdot}, u}$, and $\Phi u = \inp{\phi \sprt{\cdot}, u}$.
\paragraph{Remark on Assumption~\ref{asp:low-rank-iv}.} The existence of densities is similar to Assumption A.1 in \cite{darolles2011nonparametric}, but is stated on the conditional and marginal distributions rather than the joint distribution. Equation~\eqref{eq:iv_factorization} is akin to assuming that the conditional expectation operator $E$ admits a finite singular-value decomposition (SVD). Compactness implies the existence of a countable SVD; if the spectrum of $E$ decays sufficiently fast, a finite-dimensional approximation is meaningful \citep{ren2022latent}. This is also equivalent to the low-rank assumption used in reinforcement learning \citep{jin2020provably}.

A consequence of Assumption~\ref{asp:low-rank-iv} is that for any function $f \in L_2 \sprt{\bbP_X}$, the conditional expectation $E f$ is linear in the features $\psi$.
\begin{restatable}{proposition}{linearrangeIV} \label{prop:linear-range-IV}
    If Assumption~\ref{asp:low-rank-iv} holds, then for any function $f \in L_2 \sprt{\bbP_X}$, there exists a vector $v_f \in \bbR^d$ such that $E f = \inp{\psi \sprt{Z}, v_f}$.
\end{restatable}
By Assumption~\ref{asp:existence-sol-iv}, there exists a vector $v_0 \in \bbR^d$ such that $\bbE \sbrk{Y \given Z} = \inp{\psi \sprt{Z}, v_0}$. Thus, for a given function $f \in L_2 \sprt{\bbP_X}$, the maximizer (over $L_2 \sprt{\bbP_Z}$) in Equation~\eqref{eq:fredholm-minmax} is $g_f^\star = \inp{\psi \sprt{\cdot}, v_0 - v_f}$. This suggests the following the dual function class.
\begin{restatable}[Dual class for IV regression]{proposition}{dualIV} \label{prop:dual_iv}
    Under Assumptions~\ref{asp:existence-sol-iv}, \ref{asp:low-rank-iv}, the dual function for IV regression is realizable in
    \vspace{-3pt}
    \begin{equation} \label{eq:dual_iv}
        \calG \triangleq \scbrk{z \mapsto  \inp{\psi \sprt{z}, v}, v \in \bbR^d}\,,
    \end{equation}
    that is, we have
    \vspace{-3pt}
    \begin{equation*}
        \min_{f \in L_2 \sprt{\bbP_X}} \max_{g \in L_2 \sprt{\bbP_Z}} \calL \sprt{f, g} = \min_{f \in L_2 \sprt{\bbP_X}} \max_{g \in \calR \sprt{\Psi}} \calL \sprt{f, g}\,.
    \end{equation*}
\end{restatable}
Next, we define the class $\calF$ as follows.
\begin{restatable}[Primal space for IV]{proposition}{primalIV} \label{prop:primal_iv}
    Under Assumptions~\ref{asp:existence-sol-iv}, \ref{asp:low-rank-iv}, we have
    \vspace{-3pt}
    \begin{equation*}
        \min_{f \in L_2 \sprt{\bbP_X}} \max_{g \in \calR \sprt{\Psi}} \calL \sprt{f, g} = \min_{f \in \calR \sprt{\Phi}} \max_{g \in \calR \sprt{\Psi}} \calL \sprt{f, g}\,,
    \end{equation*}
    and the class $\calF$ defined as follows contains a solution to Problem~\ref{eq:fredholm_iv},
    \vspace{-3pt}
    \begin{equation}\label{eq:primal_iv}
        \calF \triangleq \scbrk{x \mapsto \inp{\phi \sprt{x}, u}, u \in \bbR^d}\,.
    \end{equation}
\end{restatable}
Compared to parametrizations of $f$ using RKHS functions or deep neural networks~\citep{dai2017learning,muandet2020dual,liao2020provably}, where realizability is assumed to hold, we consider a structural assumption that allows us to characterize explicitly realizable function classes.

% Moreover, we have the following relationship between the parameters of the primal variable $f$ and the corresponding optimal dual $u^*$ for estimator~\eqref{eq:fenchel_dual}.
% \begin{proposition}
% \label{prop:primal_dual_relationship}
%     Given $f = \Phi v \in \calF$, its corresponding optimal dual function of~\eqref{eq:fenchel_dual} can be represented as 
%     \vspace{-1mm}
%     \begin{equation*}
%         u_f^* \sprt{z} = \inp{\psi \sprt{z}, w_v^*} \in \calU,
%     \end{equation*}
%     with
%     \vspace{-1mm}
%     \begin{equation*}
%         \inp{\psi \sprt{z}, w_v^*} = \bbE \sbrk{y \given z} - \inp{\psi \sprt{z}, \bbE \sbrk{\phi \sprt{x} \phi \sprt{x}\transpose} v}.
%     \end{equation*}
% \end{proposition}
% \vspace{-2mm}
% \begin{proof}
% This follows immediately from plugging the parametrization of $f$ into~\eqref{eq:cme},
% \begin{align*}
%     \bbE \sbrk{f \sprt{x} \given z} = \inp{\psi \sprt{z}, \int \phi \sprt{x} \inp{\phi \sprt{x}, v} \diff P_x} = \inp{\psi \sprt{z}, \bbE \sbrk{\phi \sprt{x} \phi \sprt{x}\transpose} v}. 
% \end{align*}
% We conclude using the closed-form of $u_f^*$.
% \begin{align*}
%     w_v^* = \mathbb{E}_x \left[\phi(x) \phi(x)^\top \right] v.
% \end{align*}
% \end{proof}
%
\paragraph{Remark on the parametrizations.} The characterization of the primal function class $\calF$ through the spectral decomposition of $E$ has been exploited in IV regression~\citep{darolles2011nonparametric,wang2022spectral} and reinforcement learning~\citep{jin2020provably,yang2020reinforcement,ren2022spectral}. Specifically, the feature map $\phi$ is used as in~\propref{prop:primal_iv}. Critically, we also exploit $\psi$ to characterize the dual function class $\calG$ in~\propref{prop:dual_iv} and use it within the saddle-point formulation.

Furthermore, \citealp{darolles2011nonparametric} approximate the conditional operator by estimating unconditional densities with kernel density estimators, then taking their ratio. Instead, we propose a representation learning algorithm (Section~\ref{sec:spec_causal}) to directly learn the feature maps $\phi$ and $\psi$ under Assumption~\ref{asp:low-rank-iv}, which provides an estimate of the density ratio. We next generalize these results to the IV-OC and PCL settings, returning to the optimization problem and algorithm in Section~\ref{sec:spec_causal}.

% \Tongzheng{TODO: Consider if we need to mention this: Note that, from Equation \eqref{eq:cme}, we have that, for a given $v$, the corresponding optimal $w$ satisfies:
% \begin{align*}
%     w^* = \mathbb{E}_x \left[\phi(x) \phi(x)^\top \right] v
% \end{align*}
% which can further simplify the computation.}

% \Tongzheng{Updated: In fact, the base measure is important! It determines which function we really learn! It's not about identification! For a given $f$, we measure with different $P(z)$ and indeed different $P(x)$ will lead to a different function we can learn!}

% \Antoine{I agree that we can only identify $f$ on $\mathcal{R} \left( \Phi^\star \right)$, where $\Phi^\star: f \in L_2 \left( P_x \right) \mapsto \int \phi \left( x \right) f \left( x \right) \mu \left( d x \right)$, but I would say the argument for to justify why it is fine is to say that when we decompose $f = \left\langle u_f, \varphi \left( \cdot \right) \right\rangle + f^\perp$, with $f^\perp \in \mathcal{R} \left(\Phi^\star \right)^\perp$, the orthogonal component only contributes to the objective function with a factor $\| f^\perp \|_{L_2 \left( \mu \right)}^2$, which can be set to $0$. But maybe we are saying the same thing?}

% Hence, it is sufficient to consider $f$ in Equation \eqref{eq:primal_dual_iv} with the form $f = v^\top \phi(x)$. 

%%%%%%%%------------------------------------------------
\subsection{Instrument Variable Regression  with Observable Confounding, Proxy Causal Learning}
%%%%%%%%------------------------------------------------

% \Antoine{Something I need to think about: if we end up having problems with the current assumptions (e.g. in PCL the conditional operator is not compact), we may consider making assumptions on the conditional mean embeddings instead (which is what we care about anyway?)}

% We start from the instrument variable regression with observable confounder~(IV-OC). As we explained in~\secref{sec:prelim}, the PCL and IV with observable confounder shares the same Fredholm integral equation, the spectral representation can also be generalized to PCL. 

In this section, the operator $E: L_2 \sprt{\bbP_{X O}} \to L_2 \sprt{\bbP_{Z O}}$ is defined for any $f$ as $E f = \bbE \sbrk{f \sprt{X, O} \given Z, O}$, and we assume Problem~\ref{eq:fredholm_ivo} has a solution.
\begin{assumption} \label{asp:existence-sol-ivoc}
    There exists $f_0 \in L_2 \sprt{\bbP_{X O}}$ such that $E f_0 = \bbE \sbrk{Y \given Z, O}$.
\end{assumption}
Due to the additional dependence on the observed confounder $O$, we introduce a different low-rank assumption.
\begin{assumption} \label{asp:low-rank-ivoc}
    The distributions $\bbP_X$ and $\bbP_{X \given Z=z, O=o}$, for any $z, o$, have densities with respect to the Lebesgue measure, denoted $p_X$ and $p_{X \given Z, O} \sprt{\cdot \given z, o}$, respectively. Furthermore, there exist feature maps $\phi: \calX \to \bbR^{d_x}$, $\psi: \calZ \to \bbR^{d_z}$, and $V: \calO \to \bbR^{d_x \times d_z}$ such that for any $x, z, o$,
    \begin{equation} \label{eq:ivoc_factorization}
        p_{X \given Z, O} \sprt{x \given z, o} = p_X \sprt{x} \inp{\phi \sprt{x}, V \sprt{o} \psi \sprt{z}}\,.
    \end{equation}
\end{assumption}

% \vspace{-2mm}
% \paragraph{Remark (existence of~\eqref{eq:ivoc_factorization}):} Similarly to the SVD for condition operator~\eqref{eq:iv_factorization} in IV, we can view~\eqref{eq:ivoc_factorization} as a Tucker decomposition of $\frac{P \sprt{x \given z, o}}{P \sprt{x}}$ potentially infinite. Specifically, denote $\times_i$ as tensor-vector product for $i$-th axis, the Tucker decomposition leads to
% \begin{multline}
%     \frac{P(x|z, o)}{P(x)} = T\times_1 \phi\rbr{x}\times_2 \psi(z)\times_3\xi(o)
%     = \left\langle \phi(x), \underbrace{\rbr{T\times_3 \xi(o)}}_{V(o)}\psi(z) \right\rangle
% \end{multline}
% % by considering $V(o) = T\otimes_3 \mu(o)$
% with $T \in \RR^{d_x \times d_z\times d_o}$ and $\mu(o)\in \RR^{d_o}$. We also consider finite dimension in discussion for simplicity, but we emphasize the representation results are also applicable to infinite dimension case.

Under Assumption~\ref{asp:low-rank-ivoc}, the operator $E$ ``linearizes'' any input function $f \in L_2 \sprt{\bbP_{X O}}$ in the sense that $E f$ is linear in $\psi$ and $V$ (separately, but not jointly).
\begin{restatable}{proposition}{linearrangeIVOC} \label{prop:linear-range-IVOC}
    If Assumption~\ref{asp:low-rank-ivoc} holds, then for any function $f \in L_2 \sprt{\bbP_{X O}}$, there exists a function $v_f: \calO \to \bbR^{d_z}$ such that $E f = \inp{\psi \sprt{Z}, v_f \sprt{O}}$.
\end{restatable}
% \begin{align*}
%     E f &= \int_\calX P \sprt{x \given z, o} f \sprt{x, o} \diff x \\
%     % = & \left\langle \psi(z), V(o)^\top\int \phi(x) f(x, o) d P_{x|o}  \right\rangle
%     &=  \inp{\psi \sprt{z}, V \sprt{o}^\top \int \phi \sprt{x} f \sprt{x, o} \diff P_x}.
% \end{align*}
Fixing $O = o$, we can apply the same argument as in the proof of Proposition~\ref{prop:primal_iv} to the integral $\int_\calX \phi \sprt{x} f \sprt{x, o} \bbP_X \sprt{\diff x}$. This shows that for any $o$, it suffices to consider $f \sprt{\cdot, o}$ in the span of $\Phi$. Therefore, there exists a function $u: \calO \to \bbR^{d_x}$ such that for any $x, o$, we have $f \sprt{x, o} = \inp{\phi \sprt{x}, u \sprt{o}}$. However, it is still unclear what the function $u$ looks like.

In general, we cannot determine the representation for $u$ solely through $E$, as we did in the IV regression setting. To see this, consider the special case where the function $f$ does not depend on the treatment, \ie for any $x, o$, $f \sprt{x, o} = h \sprt{o}$ for some function $h \in L_2 \sprt{\bbP_O}$. Then, we have that $\bbE \sbrk{h \sprt{O} \given Z, O} = h \sprt{O}$, and the operator $E$ provides no information about $h$ (and thus, no information about $u$).

Given a solution to Equation~\eqref{eq:fredholm_ivo}, we have $\bbE \sbrk{f \sprt{X, O} \given Z, O} = \bbE\sbrk{Y \given Z, O}$. This implies that (i) the space for $u \sprt{o}$ should also lie in the space of $o$ in RHS, and (ii) the space for $z$ on both sides should be the same. Thus, we make the following assumption.
\begin{assumption} \label{asp:low-rank-ivoc-y}
    The distributions $\bbP_Y$ and $\bbP_{Y \given Z=z, O=o}$, for any $z, o$, have densities with respect to the Lebesgue measure, denoted $p_Y$ and $p_{Y \given Z, O} \sprt{\cdot \given z, o}$, respectively. Furthermore, there exist feature maps $\nu: \calY \to \bbR^{d_y}$, and $W: \calO \to \bbR^{d_y \times d_z}$ such that for any $y, z, o$,
    \begin{equation} \label{eq:ivoc_factorization_y}
        p_{Y \given Z, O} \sprt{y \given z, o} = p_Y \sprt{y} \inp{\nu \sprt{y}, W \sprt{o} \psi \sprt{z}}\,,
    \end{equation}
    where $\psi$ is the feature map from Assumption~\ref{asp:low-rank-ivoc}.
\end{assumption}

%factorization on $P(y|z, o)$:
%\begin{align}\label{eq:y_decomp}
%    P(y|z, o) = \left\langle P(y) \nu(y), W(o)\psi(z)\right\rangle,
%\end{align}
%where $\nu: \Ycal\rightarrow \RR^{d_y}$ and $W:\Ocal\rightarrow\in \RR^{d_y\times d_z}$. The $\psi\rbr{z}$ in~\eqref{eq:y_decomp} is shared with~\eqref{eq:x_decomp}, due to the observation {\bf ii)} aforementioned. With the same argument as the existence of~\eqref{eq:x_decomp}, the factorization~\eqref{eq:y_decomp} also exists. 

% Another problem is that, $\int \phi(x)\phi(x)^\top dP_{x|o}$ also depends on $o$. To solve this, we can further assume that
% \begin{align}
%     P(x|o) = \langle P(x)\varphi(x), \mu(o)\rangle.
% \end{align}
% Now we can write
% \begin{align}
%     \int \phi(x)\phi(x)^\top dP_{x|o} = {[P(x)\phi(x)\otimes \phi(x) \otimes \varphi(x)]^{jk}}_l \mu^l(o),
% \end{align}
% \Tongzheng{If we use Kronecker product here, we need to first vectorize first two dimension, then do the matrix multiplication, and reshape...} where we use the Einstein notation. \Bo{is it possible we use kronecker notation, instead of Einstein?} 
% Now we have that
% \begin{align*}
%     & {[V(o)^\top]^{i}}_j {[P(x)\phi(x)\otimes \phi(x) \otimes \varphi(x)]^{j}}_{kl} \mu^l(o) w^k(o)\\
%     = & V(o)^\top \int y \nu(y) dP_y.
% \end{align*}
% And we can find the representation of $o$ or solve $v(o)$ through the equation, and solve the min-max optimization problem to identify $f(x, o)$ and $u(z, o)$. 

% \paragraph{Remark:}
With this additional assumption, let us denote $Q \sprt{o} \triangleq \sprt{V \sprt{o} \transpose}^\upplus W \sprt{o}\transpose$ for any $o$. We then obtain the following class.
\begin{restatable}[Primal space for IV-OC]{proposition}{primalIVOC} \label{prop:primal-ivoc}
    Under Assumptions~\ref{asp:existence-sol-ivoc}, \ref{asp:low-rank-ivoc}, and \ref{asp:low-rank-ivoc-y}, the following class contains solutions to Equation~\ref{eq:fredholm_ivo},
    \begin{equation*}
        \Fcal \triangleq \scbrk{\sprt{x, o} \mapsto \inp{\phi \sprt{x}, B Q \sprt{o} \beta}, B \in \bbR^{d_x \times d_x}, \beta \in \bbR^{d_y}}\,.
    \end{equation*}
\end{restatable}

For any $f$, we still have a closed form for the maximizer, \ie $g_f^\star \sprt{z, o} = \bbE \sbrk{Y - f \sprt{X, O} \given Z=z, O=o}$, so we can follow the same argument than earlier.
% We first note that the Fredholm equation~\eqref{eq:fredholm_ivo} implies that $\EE\sbr{y|z, o}$ lies in the same space of $\EE\sbr{f(x, o)| z, o}$, in terms of $(z, o)$. 
% Similarly to the IV argument, we exploit the fact that given $f$, the corresponding optimal dual function of~\eqref{eq:fenchel_dual} is
% \begin{align}
%    u^*(z, o) = \EE\sbr{y|z, o} - \EE\sbr{f(x, o)| z, o}, 
%    % & = \EE\sbr{y|z, o} - \left\langle \psi(z), V(o)^\top\rbr{\int \phi(x) \phi(x)^\top  d P_{x}} v(o)  \right\rangle \nonumber \\
%     % & = \EE\sbr{y|z, o} - \inner{\psi\rbr{z}}{V(o)^\top Q(o) \alpha},
% \end{align}
% which implies the space of dual sharing the space of $\EE\sbr{y|z, o}$ and $\EE\sbr{f(x, o)|z, o}$. 
% For the estimator~\eqref{eq:fredholm-optim-reg}, we also follow the same argument as in IV: the inner product structure in optimization
% for $f\in \Fcal$, the inner product structure in optimization, \ie, 
% \begin{align*}
%     \EE\sbr{u \sprt{z, o} \cdot \bbE \sbrk{ { \sprt{y - f \sprt{x, o}}  } |z, o} }
%     = \int u\rbr{z, o} \sbr{\inner{W(o)\psi\rbr{z}}{\underbrace{\int \nu\rbr{y}ydP_y}_{\wtil\in\RR^{d}}} - \inner{V(o)\psi\rbr{z}}{}}dP_z,
%     % = \rbr{\int u(z)  \psi\rbr{z}dP_z }^\top \wtil.
% \end{align*}
% implies that the effective component of the dual also lies in the space of  $\EE\sbr{y|z, o}$ and $\EE\sbr{f(x, o)|z, o}$. Formally, we obtain
\begin{restatable}[Dual space for IV-OC]{proposition}{dualIVOC} \label{prop:dual-ivoc}
    Under Assumptions~\ref{asp:existence-sol-ivoc}, \ref{asp:low-rank-ivoc}, and \ref{asp:low-rank-ivoc-y}, the dual function of IV-OC is realizable in
    \begin{equation*}
        \calG \triangleq \scbrk{\sprt{z, o} \mapsto  \inp{\psi \sprt{z}, V \sprt{o}\transpose Q \sprt{o} \gamma}, \gamma\in \RR^{d_y}}.
    \end{equation*}
\end{restatable}
%As with~\propref{prop:primal_dual_relationship}, we can also derive the relationship between the primal and dual parametrization in IV-OC, which we omit here.

\paragraph{Remark (Representation for PCL):} IV-OC and PCL share the same conditioning structure. Thus, the representation characterization for IV-OC also applies to PCL, yielding the following primal and dual function classes
\begin{align} \label{eq:pcl_parametrization}
    \Fcal & \triangleq \Big\{\sprt{x, w} \mapsto \phi \sprt{w}\transpose B Q \sprt{x} \beta, \nonumber \\
    &\qquad\qquad\qquad B \in \bbR^{d_w \times d_w}, \beta \in \bbR^{d_x}\Big\}, \nonumber \\
    \calG & \triangleq \scbrk{\sprt{x, z} \mapsto \psi \sprt{z}\transpose V \sprt{x}\transpose Q \sprt{x} \gamma, \gamma\in \bbR^{d_x}}\,,
\end{align}
where the spectral representations arise from similar factorizations
\begin{align} \label{eq:pcl_decomp}
    p_{W \given X, Z} \sprt{w \given x, z} &= p_W \sprt{w} \inp{\phi \sprt{w}, V \sprt{x} \psi \sprt{z}}, \notag \\
    p_{Y \given X, Z} \sprt{y \given x, z} &= p_Y \sprt{y} \inp{\nu \sprt{y}, Q \sprt{x}\transpose V \sprt{x} \psi \sprt{z}}\,.
\end{align}

\vspace{-2mm}
\paragraph{Remark (Connection to existing IV-OC and PCL parametrization):} In~\citep{deaner2018proxy,mastouri2021proximal, xu2021dfpv}, the parametrization for $f$ is $f \sprt{x, w} = \zeta\transpose \sprt{\theta \sprt{x} \otimes \phi \sprt{w}}$, where $\theta$ and $\phi$ either fixed feature dictionaries or learned neural net feature dictionaries. This parametrization shares some similarity to Equation~\eqref{eq:pcl_parametrization}, if we rewrite it as
\begin{align*}
    f \sprt{x, w} &= \phi \sprt{w}\transpose B Q \sprt{x} \beta \\
    &= \inp{\beta \text{vec} \sprt{B}, Q \sprt{x}\transpose \otimes \phi \sprt{w}\transpose}\,.
\end{align*}
However, we obtain the features $Q(x)$ and $\phi \sprt{w}$ from a spectral perspective, differing from prior work.

\vspace{-2mm}
\paragraph{Remark (Comparison to~\cite{wang2022spectral}):} \cite{wang2022spectral} also exploit the spectral structure for causal inference, with a representation for IV regression that follows~\cite{darolles2011nonparametric}. However, their spectral representation for IV-OC and PCL differs. They simply reduce the IV-OC and PCL to standard IV by augmenting the treatments and instrumental variables, while we explicitly characterize the function classes for the IV-OC and PCL min-max problems and provide practical algorithms. As noted in~\cite{xu2020dfiv}, reducing to IV via augmentation ignores the problem structure, potentially leading to unnecessary complexity.

\section{CAUSAL ESTIMATION WITH SPECTRAL REPRESENTATION}\label{sec:spec_causal}
%%%%%%%%%%%%%%%%%%%%%%%%%%%%%%%%%%%%%%%%%%%%%%%%%%%%%%%%%%%%%%%%%%%%%%%%%%

In this section, we introduce empirical algorithms based on our spectral representation.

%%%%%%%%------------------------------------------------
\subsection{Contrastive Representation Learning}
%%%%%%%%------------------------------------------------

As discussed in the previous section, we can obtain a representation of the covariates by decomposing a particular conditional expectation operator. We exploit different contrastive learning objectives to implement the factorization, which is amenable to a neural network parametrization and stochastic gradient descent, as we illustrate below.

Consider the factorization of $p_{X \given Z}$ in Assumption~\ref{asp:low-rank-iv} as an example. If we want to learn mappings $\phi$ and $\psi$ such that
\begin{equation*}
    p_{X \given Z} \sprt{x \given z} = p_X \sprt{x} \inp{\phi \sprt{x}, \psi \sprt{z}},
\end{equation*}
one method is to consider a set of possible representations $\calS$ and maximize the following objective
\begin{align} \label{eq:contrastive_l2}
   \max_{\sprt{\tilde\phi, \tilde\psi} \in \calS} \calL_{\mathrm{rep}} \sprt{\tilde\phi, \tilde\psi} \triangleq \frac2n \sum_{i=1}^n \tilde\phi \sprt{x_i}\transpose \tilde\psi \sprt{z_i} & \notag\\
   - \frac{1}{n \sprt{n-1}} \sum_{i \neq j} \sprt{\tilde\phi \sprt{x_i}\transpose \tilde\psi \sprt{z_j}}^2 - 1 &\,,
\end{align}
which has been used in \citet{wang2022spectral}. Another choice is to minimize the following objective
\begin{align} \label{eq:contrastive_mle}
    \min_{\sprt{\tilde\phi, \tilde\psi} \in \calS} \calL_{\mathrm{rep}} \sprt{\tilde\phi, \tilde\psi} \triangleq -\frac1n \sum_{i=1}^n \log \tilde\phi \sprt{x_i}\transpose \tilde\psi \sprt{z_i}& \notag\\
    + \frac1n \sum_{i=1}^n \log \sprt{\sum_{i \neq j} \tilde\phi \sprt{x_i}\transpose \tilde\psi \sprt{z_j}}&\,,
\end{align}
which has been used in \citet{zhang2022making, qiu2022contrastive}. Under mild assumptions such as realizability, both methods provide a consistent estimation of the mapping $\sprt{x, z} \mapsto \phi \sprt{x}\transpose \psi \sprt{z}$, but with different theoretical guarantees. This method learns the ratio $p_{X \given Z} \sprt{x \given z} / p_X \sprt{x}$ in the form of $\inp{\phi \sprt{x}, \psi \sprt{z}}$. Another benefit of the contrastive loss is that it is naturally compatible with stochastic gradient descent~(SGD) and can thus be scaled up to large datasets.

Similarly, we can construct corresponding contrastive losses for implementing the low-rank factorization of $p_{X \given Z, O}$ and $p_{Y \given Z, O}$ for IV-OC as~\eqref{eq:ivoc_factorization} and~\eqref{eq:ivoc_factorization_y}, and $p_{W \given X, Z}$ and $p_{Y \given X, Z}$ as~\eqref{eq:pcl_decomp} for PCL, respectively.  For instance, the contrastive loss $\calL_{\mathrm{rep}}$ in \eqref{eq:contrastive_mle} for $p_{X \given Z, O}$ in IV-OC takes the form:
\begin{align*} \label{eq:contrastive_mle_ivoc}
    \calL_{\mathrm{rep}} \sprt{\tilde\phi, \tilde\psi, \tilde V} \triangleq -\frac1n \sum_{i=1}^n \log \rbr{\tilde\phi \sprt{x_i}\transpose \rbr{\tilde V(o_i)\tilde\psi \sprt{z_i}}}& \notag\\
    + \frac1n \sum_{i=1}^n \log \sprt{\sum_{i \neq j} \tilde\phi \sprt{x_i}\transpose \rbr{\tilde V(o_j)\tilde\psi \sprt{z_j}}}&\,.
\end{align*}

\paragraph{Remark(Connection to DFIV~\citep{xu2020dfiv} and DFPV~\citep{xu2021dfpv}):}
The spectral representation for IV is equivalent to the target deep feature in~\citep{xu2020dfiv}. Their deep features $\psi \sprt{x}$ and $\phi \sprt{z}$ are obtained to fulfill the condition
\begin{equation} \label{eq:deep_fea}
    \bbE \sbrk{\psi \sprt{X} \given Z = z} = A \phi \sprt{z},
\end{equation}
where $A \in \bbR^{d \times p}$ is some constant matrix independent of $x$ and $z$. Naturally, the feature maps $\phi$ and $\psi$ from the low-rank decomposition of $p_{X \given Z}$ (Assumption~\ref{asp:low-rank-iv}) provide one solution to~\eqref{eq:deep_fea} with $A = I_{d \times p}$ by definition of the SVD. The algorithm proposed by~\citet{xu2020dfiv} employs an additional $d \times d$ matrix besides the mappings $\sprt{\phi, \psi}$ through a bi-level optimization. However, it requires propagating gradients through a Cholesky decomposition, increasing the computational cost. Likewise, the representations $\phi$ of $W$, $\theta$ of $X$, and $\psi$ of $Z$ in DFPV~\citep{xu2021dfpv} are learned through bi-level optimization, with the same computational drawbacks.

% \paragraph{Representation Learning with Observable Confounding}
% Note that, factorizing Equation~\eqref{eq:x_decomp} requires the sample from $P(x|o)$, which is not feasible in practice. However, note that, for the instrument variable regression with observable confounding, we have that
% \begin{align}
%     P(x, o|z) = P(x|z, o) P(o) = \langle P(x, o) \phi(x), V(o)\psi(z) \rangle,
% \end{align}
% we can instead factorize $P(x, o|z)$ and lead to the identical feature for $x$, $o$ and $z$.
% \Tongzheng{@Bo: This part can be tricky for PCL, we also need to factorize $P(o|z)$}

\subsection{Estimation of Primal and Dual Variables}

% Given the spectral representation that learned from the conditional operators for characterizing the primal space and dual space, we can ensure the existence of solution, which are usually assumed in literature~\citep{bennett2023minimax,li2024regularized}. 
% Therefore, we can identify $f$ through the primal-dual formulation~\citep{dai2017learning,dikkala2020minimax,liao2020provably} with the revealed function structure in~\secref{sec:method} upon the learned representations. 

With the \emph{exact} spectral representation, we know which function classes to consider for both the primal and dual variables. By construction, they satisfy the assumptions required in previous works~\citep{bennett2023minimax, li2024regularized}, and thus our algorithm enjoys strong statistical guarantees, at least for IV regression.

In practice, this may not hold for two reasons: (i) the representations are \emph{learned from the data} and (ii) the user may choose a representation dimensionality $d'$ that differs from the true (potentially infinite) dimensionality $d$.

The first issue induces a bias because the learned representations are evaluated on the same data used for their training. This bias can be mitigated by splitting the data into separate sets for representation learning phase and saddle-point optimization. Alternatively, if data splitting is not feasible, a uniform covering argument can be used to show the additional error introduced by using the same data scales with the complexity of the function classes used in representation learning (\eg the number of parameters in the neural networks) divided by the square root of the number of samples. This resulting error is often reasonable, especially when compared to the initial double-sampling bias.

The second issue, where $d' \neq d$, leads to a misspecification error. If $d' < d$\footnote{Note that the case where $d' > d$ is typically handled by the regularizer, which would tend to select a solution with a small norm, effectively setting the components corresponding to the extra dimensions to zero.}, this error depends on the magnitude of the ``missing'' eigenvalues of the conditional expectation operator. If the operator's spectrum decays sufficiently rapidly, this error will be small. A rigorous analysis of these errors is left for future work; however, similar issues have been addressed in prior work (\eg \cite{wang2022spectral} in the context of representation learning, and \cite{bennett2023minimax} regarding the misspecification error).

% To ensure the existence of solution, we extract redundant spectral representations with higher dimensions, \ie, $d'_x\ge d_x$, $d'_y \ge d_y$ and $d_z'\ge d_z$. Although the enlarged function spaces guarantee the existence of solutions, it will lead to potential undetermined setting, \ie the solution is not unique. 
% Fortunately, the regularization $\Omega\rbr{f}$ in the primal-dual framework~\eqref{eq:fredholm-minmax} will automatically eliminate the ambiguity and recover the identifiability, even with the redundant spectral representation.

We now present our algorithms below.
\paragraph{Instrument variable regression.} Given the estimated representation $\widehat{\phi}$ and $\widehat{\psi}$ obtained by factorizing $p_{X \given Z}$ using Equation~\eqref{eq:contrastive_l2} or~\eqref{eq:contrastive_mle}, we  solve the following saddle-point optimization problem,
\begin{align} \label{eq:linear_primal_dual_iv}
    \min_{u \in \bbR^d} \max_{v \in \bbR^d}\, &\bbE \bigg[\sprt{v\transpose \widehat{\psi} \sprt{Z}}\sprt{Y - u\transpose \widehat{\phi} \sprt{X}} \notag\\
    &- \frac12 \sprt{v\transpose \widehat{\psi} \sprt{Z}}^2\bigg] + \lambda \Omega \sprt{u\transpose \widehat{\phi} \sprt{\cdot}}\,.
\end{align}
This yields parameters $\wh{u}$ and $\wh{v}$. We then define the estimated functions $f: x \mapsto \wh{u}\transpose \widehat{\phi} \sprt{x}$ and $g \sprt{z} = \wh{v}\transpose \widehat{\psi} \sprt{z}$.

\paragraph{Instrument variable regression with observable confounding and PCL.} We denote $\wh{\phi}, \wh{\psi}, \wh{V}, \wh{Q}$ the learned representations for IV-OC. Unlike the IV regression setting without observable confounding, we now have an additional feature of the observable $o$. The parametrization for $f$, $f: (x, o) \mapsto \phi \sprt{x}\transpose B Q \sprt{o} \beta$, contains two parameters $B$ and $\beta$. Direct substitution of this parametrization into the saddle-point problem would induce non-convexity. To recover convexity, we consider the following reparametrization:
\begin{align*}
    \phi \sprt{x}\transpose B Q \sprt{o} \beta &= \tr \sprt{\phi \sprt{x}\transpose B Q \sprt{o} \beta} \\
    %= \tr\rbr{\rbr{\rbr{Q(o)^\top \otimes \phi\rbr{x}^\top}\text{vec}\rbr{B}}^\top\beta}
   &= \tr \Bigl(\sprt{Q \sprt{o}\transpose \otimes \phi \sprt{x}\transpose}\transpose \underbrace{\beta \text{vec} \sprt{B}\transpose}_{\triangleq G}\Bigr) \\
   &= \inp{G, Q \sprt{o}\transpose \otimes \phi \sprt{x}\transpose}_{\mathrm{F}}\,,
\end{align*}
where $\inp{\cdot, \cdot}_{\mathrm{F}}$ denotes the Frobenius inner product, and the second equality comes from a property of the Kronecker matrix-vector product
\begin{align*}
    \sprt{C \otimes D} \text{vec} \sprt{G} = \text{vec} \sprt{C G D\transpose}\,.
\end{align*}
We define $f_{G}: \sprt{x, o} \mapsto \inp{G, \wh{Q} \sprt{o}\transpose \otimes \wh{\phi} \sprt{x}\transpose}_{\mathrm{F}}$ for some $G \in \bbR^{d_y \times d_x^2}$ and $g_w: \sprt{z, o} \mapsto \wh{\psi} \sprt{z}\transpose \wh{V} \sprt{o}\transpose \wh{Q} \sprt{o} w$ for a vector $w \in \bbR^{d_y}$. Based on this reparametrization, we consider the following convex-concave saddle-point optimization problem
% \begin{align}\label{eq:linear_primal_dual_iv_observable_confounder}
% \resizebox{0.92\textwidth}{!}{$
%     \min_{G}\max_{w} ~ \lambda\Omega\rbr{ \inner{G}{\Qhat(o)^\top \otimes\widehat{\phi}\rbr{x}^\top}} + \EE_{x, y, z}\bigg[
%     % - \frac{1}{2}\rbr{\widehat{\psi}\rbr{z}^\top \Vhat(o)^\top Q(o)w}^2\\
%      \rbr{\widehat{\psi}\rbr{z}^\top \Vhat(o)^\top \Qhat(o)w} \cdot\rbr{y - \inner{G}{\Qhat(o)^\top \otimes\widehat{\phi}\rbr{x}^\top}}  \bigg] 
%      $}
% \end{align}
\begin{align} \label{eq:linear_primal_dual_iv_observable_confounder}
    \min_{G \in \bbR^{d_y \times d_x^2}} \max_{w \in \bbR^{d_y}}\, \wh{\calL} \sprt{G, w}\,,
\end{align}
where we defined
\begin{align*}
    \wh{\calL} \sprt{G, w} =
    \bbE \bigg[&g_w \sprt{Z, O} \sprt{Y - f_G \sprt{X, O}} \\
    &- \frac12 g_w \sprt{Z, O}^2\bigg] + \lambda \Omega \sprt{f_G}\,.
\end{align*}
Denoting $\wh{G}$ and $\wh{w}$ the parameters solving the saddle-point problem, we return the functions $f: \sprt{x, o} \mapsto f_{\wh{G}} \sprt{x, o}$ and $g: \sprt{z, o} \mapsto g_{\wh{w}} \sprt{z, o}$. We can perform similar derivations for the PCL setting. We provide the complete algorithms in Appendix~\ref{appendix:algorithm}.

% \Bo{
% To eliminate the duplication of $v$ in primal and dual, we take the $f(x, o) = \phi(x)^\top P W(o)^\top v$ as an example. 

% For simplicity, we denote $\phi = \phi(x)\in \RR^{d\times 1}$, $P\in \RR^{d\times k}$, $W = W(o)\in\RR^{l\times k}$, $v\in \RR^{l\times 1}$, then, we have
% \begin{eqnarray*}
%     \phi^\top P W^\top v &=& \tr\rbr{\phi^\top P W^\top v} = \tr\rbr{\rbr{ \underbrace{\rbr{W\otimes \phi^\top }}_{A\in \RR^{l\times dk}} \underbrace{\text{vec}(P)}_{p\in \RR^{dk\times 1}}}^\top  v} \\
%     &=& \tr\rbr{p^\top A^\top v} = \tr\rbr{A^\top \underbrace{vp^\top}_{\beta\in \RR^{l\times dk}}} = \tr\rbr{\beta A^\top}. 
%     % = \tr\rbr{\underbrace{{pv^\top}}_{\beta\in \RR^{dk\times l}} A} = \tr\rbr{\beta A}
% \end{eqnarray*}
% We have the equivalent parametrization for $f(x, o) = \tr\rbr{\beta \rbr{W(o)^\top \otimes \phi\rbr{x}}} = \inner{\beta }{W(o) \otimes \phi\rbr{x}^\top}$. 

% \Tongzheng{where the first equation is due to the mixed Kronecker matrix-vector product:
% \begin{align*}
%     (A\otimes B) \text{vec}(V) = \text{vec}(BVA)
% \end{align*}}
% Similarly, we have the equivalent parametrization for $u(z, o) = \psi(z)^\top V(o)^\top Q W(o)^\top v = \inner{\underbrace{v\text{vec}(Q)^\top}_{\alpha}}{W(o)\otimes \psi(z)^\top V(o)^\top }$.

% }

\section{EXPERIMENTS}\label{sec:experiments}

We evaluate the empirical performance of the proposed \algabb and \algabbpcl and several modern methods for IV  with and without observable confounders, as well as PCL~\footnote{Available at \url{https://github.com/haotiansun14/SpecIV}}. This evaluation is conducted on two datasets. Following \cite{xu2020dfiv, xu2021dfpv}, we utilize the out-of-sample mean-square error (OOS MSE) as the metric for all test cases. Experiment details and setup can be found in Appendix~\ref{app:setup}.

\paragraph{Baselines.} For the IV regression, we contrast \algabb with two methods with pre-specified features -- the Kernel IV (KIV)~\citep{singh2019kernel} and the Dual Embedding (DE)~\citep{dai2017learning}~\footnote{The algorithm Dual IV~\citep{muandet2020dual} follows the same primal-dual framework of DE, but with closed-form solution, which involves explicit matrix inverse, while DE exploits the stochastic gradient algorithm that is more computationally efficient.}. Additionally, we evaluate several approaches that leverage deep neural networks for feature representation, namely, DFIV~\citep{xu2020dfiv}, and DeepGMM~\citep{bennett2019deep}.
% We employ the PCL setting proposed in \cite{xu2021dfpv}, with further details provided in Appendix~\ref{app:data_dsprites}. 
We evaluate DFPV~\citep{xu2021dfpv}, KPV~\citep{mastouri2023kpv}, PMMR~\citep{mastouri2023kpv}, and CEVAE~\citep{louizos2017cevae}. 
\paragraph{Datasets.} We conduct experiments on the dSprites Dataset~\citep{dsprites} for both low- and high-dimensional IV and PCL settings, as well as the Demand Design dataset~\citep{hartford2017deepiv} for IV with observable confounders. 
1) \textbf{dSprites} comprises images determined by five latent parameters (\texttt{shape}, \texttt{scale}, \texttt{rotation}, \texttt{posX}, \texttt{posY}). Each image has dimensions of $64\times64=4096$ and serves as the treatment variable $X$. Following the setup in \cite{xu2020dfiv}, we keep \texttt{shape} fixed as \texttt{heart} and use \texttt{posY} as the hidden confounder. The remaining latent variables compose the instrument variable $Z$. 
In addition, we introduce a high-dimensional setting where the instruments are mapped to a high-dimensional variable in $\mathbb{R}^{2352}$, as proposed in \cite{bennett2019deep}. 
% Further details regarding the data generation process can be found in Appendix~\ref{app:data_dsprites}.
2) \textbf{Demand Design} is a synthetic benchmark for nonlinear IV regression. Given the airplane ticket price $P$, the objective is to predict ticket demands, denoted as $Y$, in the presence of two observable confounders: price sensitivity $S\in\{1,\dots,7\}$ and the year time $T\in[0,10]$. Additionally, we introduce an unobservable confounder, represented as correlated noise in $P$ and $Y$. Furthermore, we set the fuel price $C$ as the instrument variable and $X$ as the treatment. We also employ the mapping function introduced in \cite{bennett2019deep} to map both $P$ and $S$ to high-dimensional variables in $\mathbb{R}^{784}$. This represents a more challenging scenario, as the IV regression method must estimate the relevant variables from noisy, high-dimensional data. Details about the data generation are presented in Appendix~\ref{app:data_mapping} and \ref{app:data_demand}.

\paragraph{Instrument Variable Regression.}

\begin{figure}[!htb]
    \centering
    \begin{subfigure}[b]{0.5\textwidth}
    \centering
        \includegraphics[width=.95\linewidth]{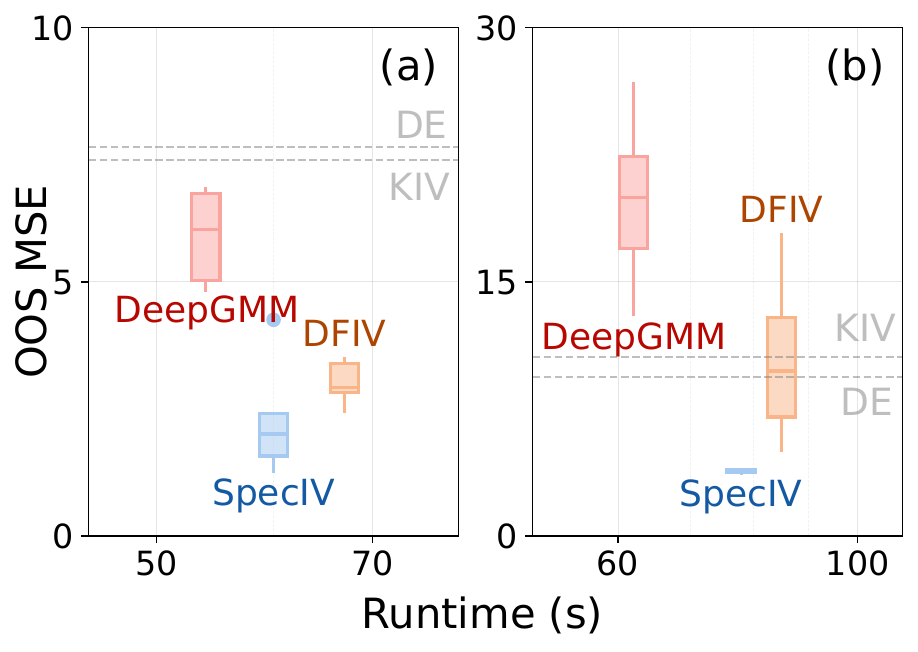}
    \end{subfigure}
    \begin{subfigure}[b]{0.5\textwidth}
    \centering
        \includegraphics[width=.95\linewidth]{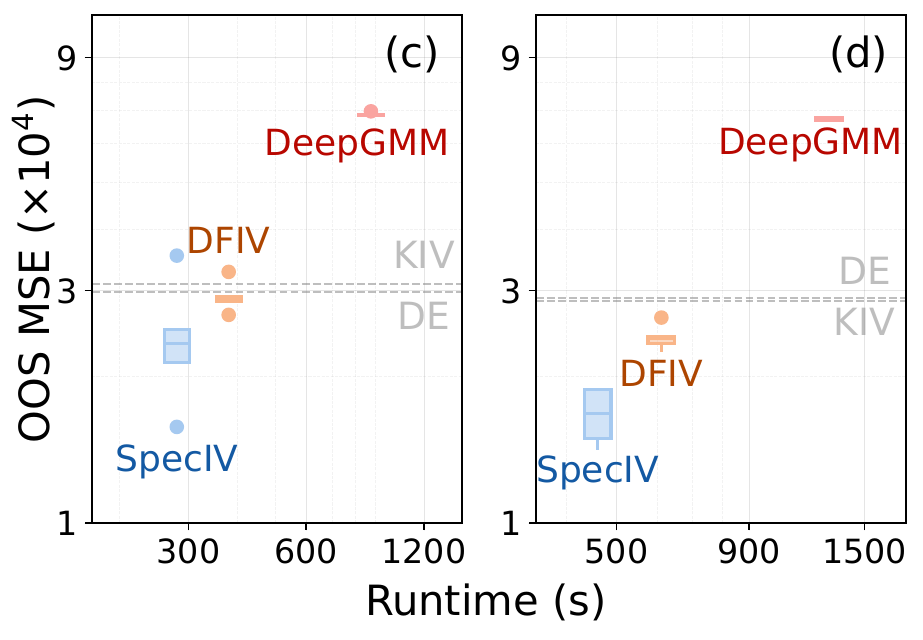}
    \end{subfigure}
    \caption{MSE and Runtime on dSprites Dataset with (a) Low-Dimensional Instruments (32) and (b) High-Dimensional Instruments (64), and Demand Design dataset with (c) 5,000 and (d) 10,000 training samples. Methods employing pre-specified features, \textit{i.e.}, DE and KIV, are represented as dashed baselines.}
    \label{fig:comparison}
\end{figure}

Figure~\ref{fig:comparison}a and \ref{fig:comparison}b depict the performance and the execution time of various methods on the dSprites dataset. The optimal performance corresponds to proximity to the bottom-left corner, indicating low MSE and reduced runtime.
\algabb outperforms all other approaches in terms of MSE in both low- and high-dimensional contexts while maintaining a competitive runtime. 
Methods with fixed features, such as KIV and DE, yield high MSE. The cause might be their reliance on predetermined feature representations, which restricts adaptability.
In low-dimensional scenarios, DeepGMM and DFIV manage to keep errors within a reasonable range, but DeepGMM's performance declines sharply with increased feature dimensions, while \algabb maintains good estimation accuracy. Furthermore, both DeepGMM and DFIV exhibit increased variance in the high-dimensional setting, possibly due to their ineffective representation learning for high-dimensional features. In summary, \algabb demonstrates superior and consistent performance over the baseline methods in both settings for instrumental variable tasks.

\paragraph{Instrument Variable Regression with Observable Confounder.}

In the Demand Design dataset, we employ two settings for the observable confounders. For \algabb and DFIV, we set the ticket price $P$ as the treatment and the fuel price $C$ as the instrument, accompanied by the observables $(T,S)$.  Each of these three components is presented using distinct neural networks. 
For other methods that don't model observables separately, we integrate the observables directly with the treatment and instrument. This means $(P,T,S)$ is used as the treatment and $(C,T,S)$ as the instrument. 
Figure~\ref{fig:comparison}c and Figure~\ref{fig:comparison}d represent the performance of these approaches given different amounts of training samples. In both scenarios, \algabb consistently delivers the lowest error also achieving the best runtime. 
DeepGMM performs the least efficiently, with the highest error rates and runtimes. This may result from incorporating observables into both treatment and instrument. This overlooks the fact that we only need to consider the conditional expectation of $X$ given $(Z,O)$, thus making the problem suboptimal.
KIV and DE, limited by their fixed feature representations, do not benefit from an increase in training sample size and remain less expressive.
Overall, the \algabb method outperforms other approaches in learning structured functions with the presence of observable confounders.

\paragraph{Proxy Causal Learning.}
\begin{figure}[!ht]
    \centering
    \includegraphics[width=.95\linewidth]{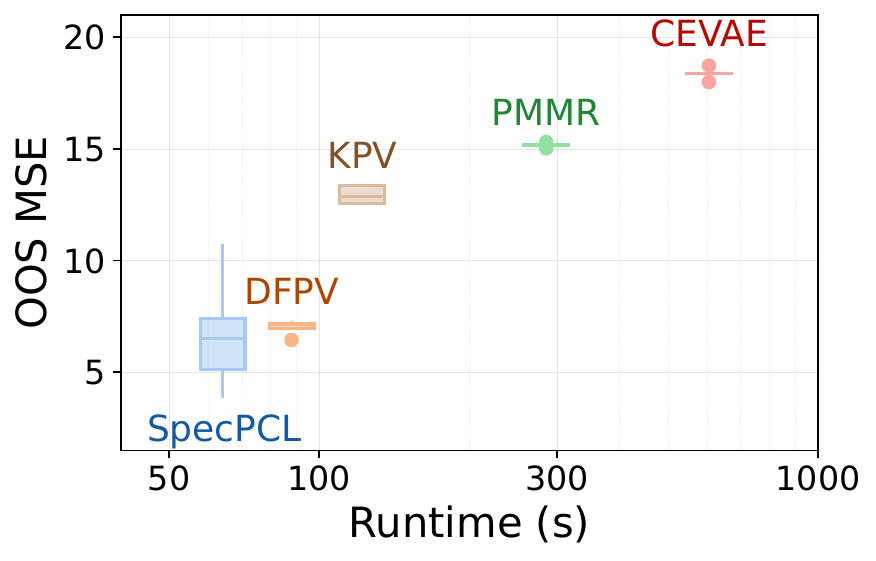}
    \caption{MSE and Runtime of structural function estimation on the dSprites dataset with 5,000 training samples.}
    \label{fig:dprites_pcl}
\end{figure}
Experimental results are shown in Figure~\ref{fig:dprites_pcl}. \algabbpcl outperforms existing methods with superior estimation accuracy and computational efficiency, which demonstrates its strong capability in capturing complex structural functions. DFPV offers a reasonable balance between error and runtime but falls short of \algabb's performance.
KPV appears to yield a slightly lower error, possibly because KPV harnesses a greater number of parameters to express its solution. 
CEVAE, despite its flexibility through neural networks, underperforms all other methods. The reason could be that CEVAE does not leverage the relation between the proxies and the structural function.

\vspace{-2pt}
\section{CONCLUSION}
\vspace{-2mm}

We introduced a novel spectral representation learning framework for causal estimation in the presence of hidden confounders. Our approach leverages a low-rank assumption on conditional densities to characterize function classes within a saddle-point optimization problem. This characterization enables the development of efficient algorithms for IV regression, both with and without observed confounders, and for proximal causal learning (PCL). Extensive experimental validation demonstrates that our method outperforms existing approaches. While our method benefits from existing theoretical guarantees for IV regression (\eg, \citealp{bennett2023source, wang2022spectral}), extending these guarantees to IV regression with observed confounders and to PCL remains a challenging open problem. Another interesting direction for future research is a formal comparison of the performance of two-stage methods and conditional moment methods.

\subsubsection*{Acknowledgements}
% All acknowledgments go at the end of the paper, including thanks to reviewers who gave useful comments, to colleagues who contributed to the ideas, and to funding agencies and corporate sponsors that provided financial support. 
% To preserve the anonymity, please include acknowledgments \emph{only} in the camera-ready papers.

Arthur Gretton and Antoine Moulin wish to thank Dimitri Meunier for helpful discussions. 
This project has received funding from the European Research Council (ERC), under the European Union’s Horizon 2020 research and innovation programme (Grant agreement No. 950180); 
from the Gatsby Charitable Foundation;
from NSF ECCS2401391;
from NSF IIS-2403240;
and from Dolby support.
Antoine Moulin acknowledges travel support from ELISE (GA no 951847) and the Gatsby Charitable Foundation.

% \subsubsection*{References}
% \begin{thebibliography}{}
% \setlength{\itemindent}{-\leftmargin}
% \makeatletter\renewcommand{\@biblabel}[1]{}\makeatother
% \end{thebibliography}
\bibliographystyle{plainnat}
\bibliography{aistats.bib}

\begin{thebibliography}{44}
\providecommand{\natexlab}[1]{#1}
\providecommand{\url}[1]{\texttt{#1}}
\expandafter\ifx\csname urlstyle\endcsname\relax
  \providecommand{\doi}[1]{doi: #1}\else
  \providecommand{\doi}{doi: \begingroup \urlstyle{rm}\Url}\fi

\bibitem[Antos et~al.(2008)Antos, Szepesv{\'a}ri, and Munos]{antos2008learning}
Andr{\'a}s Antos, Csaba Szepesv{\'a}ri, and R{\'e}mi Munos.
\newblock Learning near-optimal policies with bellman-residual minimization based fitted policy iteration and a single sample path.
\newblock \emph{Machine Learning}, 71:\penalty0 89--129, 2008.

\bibitem[Bennett et~al.(2019)Bennett, Kallus, and Schnabel]{bennett2019deep}
Andrew Bennett, Nathan Kallus, and Tobias Schnabel.
\newblock Deep generalized method of moments for instrumental variable analysis.
\newblock \emph{Advances in neural information processing systems}, 32, 2019.

\bibitem[Bennett et~al.(2023{\natexlab{a}})Bennett, Kallus, Mao, Newey, Syrgkanis, and Uehara]{bennett2023inference}
Andrew Bennett, Nathan Kallus, Xiaojie Mao, Whitney Newey, Vasilis Syrgkanis, and Masatoshi Uehara.
\newblock Inference on strongly identified functionals of weakly identified functions.
\newblock In \emph{The Thirty Sixth Annual Conference on Learning Theory}, pages 2265--2265. PMLR, 2023{\natexlab{a}}.

\bibitem[Bennett et~al.(2023{\natexlab{b}})Bennett, Kallus, Mao, Newey, Syrgkanis, and Uehara]{bennett2023minimax}
Andrew Bennett, Nathan Kallus, Xiaojie Mao, Whitney Newey, Vasilis Syrgkanis, and Masatoshi Uehara.
\newblock Minimax instrumental variable regression and $ l\_2 $ convergence guarantees without identification or closedness.
\newblock In \emph{The Thirty Sixth Annual Conference on Learning Theory}, pages 2291--2318. PMLR, 2023{\natexlab{b}}.

\bibitem[Bennett et~al.(2023{\natexlab{c}})Bennett, Kallus, Mao, Newey, Syrgkanis, and Uehara]{bennett2023source}
Andrew Bennett, Nathan Kallus, Xiaojie Mao, Whitney Newey, Vasilis Syrgkanis, and Masatoshi Uehara.
\newblock Source condition double robust inference on functionals of inverse problems.
\newblock \emph{arXiv preprint arXiv:2307.13793}, 2023{\natexlab{c}}.

\bibitem[Bradtke and Barto(1996)]{bradtke1996linear}
Steven~J Bradtke and Andrew~G Barto.
\newblock Linear least-squares algorithms for temporal difference learning.
\newblock \emph{Machine learning}, 22\penalty0 (1):\penalty0 33--57, 1996.

\bibitem[Chen and Christensen(2018)]{chen2018sieve}
Xiaohong Chen and Timothy~M. Christensen.
\newblock Optimal sup-norm rates and uniform inference on nonlinear functionals of nonparametric iv regression: Nonlinear functionals of nonparametric iv.
\newblock \emph{Quantitative Economics}, 9\penalty0 (1):\penalty0 39–84, March 2018.
\newblock ISSN 1759-7323.
\newblock \doi{10.3982/qe722}.
\newblock URL \url{http://dx.doi.org/10.3982/QE722}.

\bibitem[Dai et~al.(2017)Dai, He, Pan, Boots, and Song]{dai2017learning}
Bo~Dai, Niao He, Yunpeng Pan, Byron Boots, and Le~Song.
\newblock Learning from conditional distributions via dual embeddings.
\newblock In \emph{Artificial Intelligence and Statistics}, pages 1458--1467. PMLR, 2017.

\bibitem[Darolles et~al.(2011)Darolles, Fan, Florens, and Renault]{darolles2011nonparametric}
Serge Darolles, Yanqin Fan, Jean-Pierre Florens, and Eric Renault.
\newblock Nonparametric instrumental regression.
\newblock \emph{Econometrica}, 79\penalty0 (5):\penalty0 1541--1565, 2011.

\bibitem[Deaner(2018)]{deaner2018proxy}
Ben Deaner.
\newblock Proxy controls and panel data.
\newblock \emph{arXiv preprint arXiv:1810.00283}, 2018.

\bibitem[Dikkala et~al.(2020)Dikkala, Lewis, Mackey, and Syrgkanis]{dikkala2020minimax}
Nishanth Dikkala, Greg Lewis, Lester Mackey, and Vasilis Syrgkanis.
\newblock Minimax estimation of conditional moment models.
\newblock \emph{Advances in Neural Information Processing Systems}, 33:\penalty0 12248--12262, 2020.

\bibitem[Engl et~al.(1996)Engl, Hanke, and Neubauer]{engl1996regularization}
Heinz~Werner Engl, Martin Hanke, and Andreas Neubauer.
\newblock \emph{Regularization of inverse problems}, volume 375.
\newblock Springer Science \& Business Media, 1996.

\bibitem[Grunewalder et~al.(2012)Grunewalder, Lever, Baldassarre, Patterson, Gretton, and Pontil]{GruLevBalPatetal12}
S.~Grunewalder, G.~Lever, L.~Baldassarre, S.~Patterson, A.~Gretton, and M.~Pontil.
\newblock Conditional mean embeddings as regressors.
\newblock In \emph{International Conference on Machine Learning}, 2012.

\bibitem[Hartford et~al.(2017)Hartford, Lewis, Leyton-Brown, and Taddy]{hartford2017deepiv}
Jason Hartford, Greg Lewis, Kevin Leyton-Brown, and Matt Taddy.
\newblock Deep iv: A flexible approach for counterfactual prediction.
\newblock In \emph{International Conference on Machine Learning}, pages 1414--1423. PMLR, 2017.

\bibitem[Horowitz(2011)]{horowitz2011applied}
Joel~L Horowitz.
\newblock Applied nonparametric instrumental variables estimation.
\newblock \emph{Econometrica}, 79\penalty0 (2):\penalty0 347--394, 2011.

\bibitem[Jin et~al.(2020)Jin, Yang, Wang, and Jordan]{jin2020provably}
Chi Jin, Zhuoran Yang, Zhaoran Wang, and Michael~I Jordan.
\newblock Provably efficient reinforcement learning with linear function approximation.
\newblock In \emph{Conference on Learning Theory}, pages 2137--2143. PMLR, 2020.

\bibitem[Kompa et~al.(2022)Kompa, Bellamy, Kolokotrones, Beam, et~al.]{kompa2022deep}
Benjamin Kompa, David Bellamy, Tom Kolokotrones, Andrew Beam, et~al.
\newblock Deep learning methods for proximal inference via maximum moment restriction.
\newblock \emph{Advances in Neural Information Processing Systems}, 35:\penalty0 11189--11201, 2022.

\bibitem[Kreyszig(1991)]{kreyszig1991introductory}
Erwin Kreyszig.
\newblock \emph{Introductory functional analysis with applications}, volume~17.
\newblock John Wiley \& Sons, 1991.

\bibitem[Kuroki and Pearl(2014)]{kuroki2014measurement}
Manabu Kuroki and Judea Pearl.
\newblock Measurement bias and effect restoration in causal inference.
\newblock \emph{Biometrika}, 101\penalty0 (2):\penalty0 423--437, 2014.

\bibitem[Li et~al.(2022)Li, Meunier, Mollenhauer, and Gretton]{li2022optimal}
Zhu Li, Dimitri Meunier, Mattes Mollenhauer, and Arthur Gretton.
\newblock Optimal rates for regularized conditional mean embedding learning.
\newblock In \emph{Advances in Neural Information Processing Systems}, 2022.

\bibitem[Li et~al.(2024)Li, Lan, Syrgkanis, Wang, and Uehara]{li2024regularized}
Zihao Li, Hui Lan, Vasilis Syrgkanis, Mengdi Wang, and Masatoshi Uehara.
\newblock Regularized deepiv with model selection.
\newblock \emph{arXiv preprint arXiv:2403.04236}, 2024.

\bibitem[Liao et~al.(2020)Liao, Chen, Yang, Dai, Kolar, and Wang]{liao2020provably}
Luofeng Liao, You-Lin Chen, Zhuoran Yang, Bo~Dai, Mladen Kolar, and Zhaoran Wang.
\newblock Provably efficient neural estimation of structural equation models: An adversarial approach.
\newblock \emph{Advances in Neural Information Processing Systems}, 33:\penalty0 8947--8958, 2020.

\bibitem[Louizos et~al.(2017)Louizos, Shalit, Mooij, Sontag, Zemel, and Welling]{louizos2017cevae}
Christos Louizos, Uri Shalit, Joris~M Mooij, David Sontag, Richard Zemel, and Max Welling.
\newblock Causal effect inference with deep latent-variable models.
\newblock \emph{Advances in neural information processing systems}, 30, 2017.

\bibitem[Mastouri et~al.(2021)Mastouri, Zhu, Gultchin, Korba, Silva, Kusner, Gretton, and Muandet]{mastouri2021proximal}
Afsaneh Mastouri, Yuchen Zhu, Limor Gultchin, Anna Korba, Ricardo Silva, Matt Kusner, Arthur Gretton, and Krikamol Muandet.
\newblock Proximal causal learning with kernels: Two-stage estimation and moment restriction.
\newblock In \emph{International conference on machine learning}, pages 7512--7523. PMLR, 2021.

\bibitem[Mastouri et~al.(2023)Mastouri, Zhu, Gultchin, Korba, Silva, Kusner, Gretton, and Muandet]{mastouri2023kpv}
Afsaneh Mastouri, Yuchen Zhu, Limor Gultchin, Anna Korba, Ricardo Silva, Matt~J. Kusner, Arthur Gretton, and Krikamol Muandet.
\newblock Proximal causal learning with kernels: Two-stage estimation and moment restriction, 2023.

\bibitem[Matthey et~al.(2017)Matthey, Higgins, Hassabis, and Lerchner]{dsprites}
Loic Matthey, Irina Higgins, Demis Hassabis, and Alexander Lerchner.
\newblock dsprites: Disentanglement testing sprites dataset.
\newblock https://github.com/deepmind/dsprites-dataset/, 2017.

\bibitem[Miao et~al.(2018)Miao, Geng, and Tchetgen~Tchetgen]{miao2018identifying}
Wang Miao, Zhi Geng, and Eric~J Tchetgen~Tchetgen.
\newblock Identifying causal effects with proxy variables of an unmeasured confounder.
\newblock \emph{Biometrika}, 105\penalty0 (4):\penalty0 987--993, 2018.

\bibitem[Muandet et~al.(2020)Muandet, Mehrjou, Lee, and Raj]{muandet2020dual}
Krikamol Muandet, Arash Mehrjou, Si~Kai Lee, and Anant Raj.
\newblock Dual instrumental variable regression.
\newblock In \emph{Advances in Neural Information Processing Systems}, volume~33, pages 2710--2721. Curran Associates, Inc., 2020.

\bibitem[Newey and Powell(2003)]{newey2003instrumental}
Whitney~K Newey and James~L Powell.
\newblock Instrumental variable estimation of nonparametric models.
\newblock \emph{Econometrica}, 71\penalty0 (5):\penalty0 1565--1578, 2003.

\bibitem[Qiu et~al.(2022)Qiu, Wang, Bai, Yang, and Wang]{qiu2022contrastive}
Shuang Qiu, Lingxiao Wang, Chenjia Bai, Zhuoran Yang, and Zhaoran Wang.
\newblock Contrastive ucb: Provably efficient contrastive self-supervised learning in online reinforcement learning.
\newblock In \emph{International Conference on Machine Learning}, pages 18168--18210. PMLR, 2022.

\bibitem[Ren et~al.(2022{\natexlab{a}})Ren, Xiao, Zhang, Li, Wang, Sanghavi, Schuurmans, and Dai]{ren2022latent}
Tongzheng Ren, Chenjun Xiao, Tianjun Zhang, Na~Li, Zhaoran Wang, Sujay Sanghavi, Dale Schuurmans, and Bo~Dai.
\newblock Latent variable representation for reinforcement learning.
\newblock \emph{arXiv preprint arXiv:2212.08765}, 2022{\natexlab{a}}.

\bibitem[Ren et~al.(2022{\natexlab{b}})Ren, Zhang, Lee, Gonzalez, Schuurmans, and Dai]{ren2022spectral}
Tongzheng Ren, Tianjun Zhang, Lisa Lee, Joseph~E Gonzalez, Dale Schuurmans, and Bo~Dai.
\newblock Spectral decomposition representation for reinforcement learning.
\newblock \emph{arXiv preprint arXiv:2208.09515}, 2022{\natexlab{b}}.

\bibitem[Singh et~al.(2019)Singh, Sahani, and Gretton]{singh2019kernel}
Rahul Singh, Maneesh Sahani, and Arthur Gretton.
\newblock Kernel instrumental variable regression.
\newblock \emph{Advances in Neural Information Processing Systems}, 32, 2019.

\bibitem[Song et~al.(2009)Song, Huang, Smola, and Fukumizu]{song2009cme}
Le~Song, Jonathan Huang, Alexander Smola, and Kenji Fukumizu.
\newblock Hilbert space embeddings of conditional distributions with applications to dynamical systems.
\newblock In \emph{International Conference on Machine Learning}, pages 961 -- 968, 2009.

\bibitem[Stock and Watson(2007)]{stock2007econometrics}
James Stock and Mark Watson.
\newblock \emph{Introduction to Econometrics 2nd edition}.
\newblock Prentiss Hall, 2007.

\bibitem[Stock and Trebbi(2003)]{stock2003retrospectives}
James~H Stock and Francesco Trebbi.
\newblock Retrospectives: Who invented instrumental variable regression?
\newblock \emph{Journal of Economic Perspectives}, 17\penalty0 (3):\penalty0 177--194, 2003.

\bibitem[Sutton and Barto(2018)]{sutton2018reinforcement}
Richard~S Sutton and Andrew~G Barto.
\newblock \emph{Reinforcement learning: An introduction}.
\newblock MIT press, 2018.

\bibitem[Wang et~al.(2022)Wang, Luo, Li, Zhu, and Sch{\"o}lkopf]{wang2022spectral}
Ziyu Wang, Yucen Luo, Yueru Li, Jun Zhu, and Bernhard Sch{\"o}lkopf.
\newblock Spectral representation learning for conditional moment models.
\newblock \emph{arXiv preprint arXiv:2210.16525}, 2022.

\bibitem[Wright(1928)]{wright1928tariff}
Philip~Green Wright.
\newblock \emph{The tariff on animal and vegetable oils}.
\newblock Number~26. Macmillan, 1928.

\bibitem[Xu et~al.(2020)Xu, Chen, Srinivasan, de~Freitas, Doucet, and Gretton]{xu2020dfiv}
Liyuan Xu, Yutian Chen, Siddarth Srinivasan, Nando de~Freitas, Arnaud Doucet, and Arthur Gretton.
\newblock Learning deep features in instrumental variable regression.
\newblock \emph{arXiv preprint arXiv:2010.07154}, 2020.

\bibitem[Xu et~al.(2021)Xu, Kanagawa, and Gretton]{xu2021dfpv}
Liyuan Xu, Heishiro Kanagawa, and Arthur Gretton.
\newblock Deep proxy causal learning and its application to confounded bandit policy evaluation.
\newblock \emph{Advances in Neural Information Processing Systems}, 34:\penalty0 26264--26275, 2021.

\bibitem[Yang and Wang(2020)]{yang2020reinforcement}
Lin Yang and Mengdi Wang.
\newblock Reinforcement learning in feature space: Matrix bandit, kernels, and regret bound.
\newblock In \emph{International Conference on Machine Learning}, pages 10746--10756. PMLR, 2020.

\bibitem[Zhang et~al.(2023)Zhang, Imaizumi, Sch\"olkopf, and Muandet]{ZhangImaizumiScholkopfMuandet2023}
Rui Zhang, Masaaki Imaizumi, Bernhard Sch\"olkopf, and Krikamol Muandet.
\newblock Instrumental variable regression via kernel maximum moment loss.
\newblock \emph{Journal of Causal Inference}, 11\penalty0 (1), 2023.

\bibitem[Zhang et~al.(2022)Zhang, Ren, Yang, Gonzalez, Schuurmans, and Dai]{zhang2022making}
Tianjun Zhang, Tongzheng Ren, Mengjiao Yang, Joseph Gonzalez, Dale Schuurmans, and Bo~Dai.
\newblock Making linear mdps practical via contrastive representation learning.
\newblock In \emph{International Conference on Machine Learning}, pages 26447--26466. PMLR, 2022.

\end{thebibliography}

%%%%%%%%%%%%%%%%%%%%%%%%%%%%%%%%%%%%%%%%%%%%%%%%%%%%%%%%%%%%
\section*{Checklist}

% %%% BEGIN INSTRUCTIONS %%%
% The checklist follows the references. For each question, choose your answer from the three possible options: Yes, No, Not Applicable.  You are encouraged to include a justification to your answer, either by referencing the appropriate section of your paper or providing a brief inline description (1-2 sentences). 
% Please do not modify the questions.  Note that the Checklist section does not count towards the page limit. Not including the checklist in the first submission won't result in desk rejection, although in such case we will ask you to upload it during the author response period and include it in camera ready (if accepted).

% \textbf{In your paper, please delete this instructions block and only keep the Checklist section heading above along with the questions/answers below.}
% %%% END INSTRUCTIONS %%%

 \begin{enumerate}

 \item For all models and algorithms presented, check if you include:
 \begin{enumerate}
   \item A clear description of the mathematical setting, assumptions, algorithm, and/or model. [Yes]
   \item An analysis of the properties and complexity (time, space, sample size) of any algorithm. [Yes]
   \item (Optional) Anonymized source code, with specification of all dependencies, including external libraries. [No]
 \end{enumerate}

 \item For any theoretical claim, check if you include:
 \begin{enumerate}
   \item Statements of the full set of assumptions of all theoretical results. [Yes]
   \item Complete proofs of all theoretical results. [Yes]
   \item Clear explanations of any assumptions. [Yes]     
 \end{enumerate}

 \item For all figures and tables that present empirical results, check if you include:
 \begin{enumerate}
   \item The code, data, and instructions needed to reproduce the main experimental results (either in the supplemental material or as a URL). [Yes]
   \item All the training details (e.g., data splits, hyperparameters, how they were chosen). [Yes]
         \item A clear definition of the specific measure or statistics and error bars (e.g., with respect to the random seed after running experiments multiple times). [Yes]
         \item A description of the computing infrastructure used. (e.g., type of GPUs, internal cluster, or cloud provider). [Yes]
 \end{enumerate}

 \item If you are using existing assets (e.g., code, data, models) or curating/releasing new assets, check if you include:
 \begin{enumerate}
   \item Citations of the creator If your work uses existing assets. [Yes]
   \item The license information of the assets, if applicable. [No]
   \item New assets either in the supplemental material or as a URL, if applicable. [Not Applicable]
   \item Information about consent from data providers/curators. [Not Applicable]
   \item Discussion of sensible content if applicable, e.g., personally identifiable information or offensive content. [Not Applicable]
 \end{enumerate}

 \item If you used crowdsourcing or conducted research with human subjects, check if you include:
 \begin{enumerate}
   \item The full text of instructions given to participants and screenshots. [Not Applicable]
   \item Descriptions of potential participant risks, with links to Institutional Review Board (IRB) approvals if applicable. [Not Applicable]
   \item The estimated hourly wage paid to participants and the total amount spent on participant compensation. [Not Applicable]
 \end{enumerate}

 \end{enumerate}

\newpage
\appendix
\onecolumn

%%%%%%%%%%%%%%%%%%%%%%%%%%%%%%%%%%%%%%%%%%%%%%%%%%%%%%%%%%%%
\section{LIMITATIONS AND BROADER IMPACTS}
%%%%%%%%%%%%%%%%%%%%%%%%%%%%%%%%%%%%%%%%%%%%%%%%%%%%%%%%%%%%

\subsection{Limitations}
\label{app:limitation}

In this paper, we introduced a novel spectral representation learning framework for causal estimation in the presence of hidden confounders. We demonstrated the superior performance of our method over existing approaches, showing reduced computational burden and increased accuracy. While our algorithm benefits from established theoretical guarantees for instrumental variables (IV) (e.g., \citealp{bennett2023source, wang2022spectral}), our advantages for IV with observed confounders and principal component learning (PCL) are primarily justified empirically. Future work will extend existing theoretical proofs to these settings.

\subsection{Broader Impacts}
\label{app:broader_impact}

\textbf{Potential Positive Societal Impacts.} Accurately estimating causal effects has significant potential to benefit society across multiple disciplines. For example, understanding causal relationships in epidemiology can enhance public health strategies, contributing to better disease prevention and control. This paper addresses the problem of estimating causal effects with hidden confounders. The proposed method can be effectively applied to various domains with different settings. This leads to more informed decision-making and improved outcomes in various fields.

\textbf{Potential Negative Societal Impacts.} While the proposed method for estimating causal effects offers multiple benefits, there are potential negative societal impacts to consider. Misapplication of the method without thorough validation could lead to incorrect conclusions and potentially harmful decisions, particularly in sensitive areas like healthcare or public policy.
%%%%%%%%%%%%%%%%%%%%%%%%%%%%%%%%%%%%%%%%%%%%%%%%%%%%%%%%%%%%
\section{COMPLETE ALGORITHMS}
\label{appendix:algorithm}
%%%%%%%%%%%%%%%%%%%%%%%%%%%%%%%%%%%%%%%%%%%%%%%%%%%%%%%%%%%%

% \subsection{The Complete Algorithm}
To make our whole procedure clear, we provide the complete algorithm framework for causal estimation with spectral representation, with modifications for corresponding problems.

\begin{algorithm}
    \caption{Spectral Representation Learning for Instrument Variable Regression}
    \label{alg:iv}
    \begin{algorithmic}
        \INPUT Function class $\mathcal{F}$ of the representation.
        \STATE Estimate $\phi: \calX \to \bbR^d$ and $\psi: \calZ \to \bbR^d$ with samples $\scbrk{x_i, z_i}$ using Equation~\eqref{eq:contrastive_l2} or Equation~\eqref{eq:contrastive_mle} to obtain the estimated representations $\widehat{\phi}$ and $\widehat{\psi}$.
        \STATE Solve the min-max problem Equation~\eqref{eq:linear_primal_dual_iv} to obtain the parameters $v$ and $w$.
        \OUTPUT $f: x \mapsto \inp{v, \widehat{\phi} \sprt{x}}$ and $g: z \mapsto \inp{w, \wh{\psi} \sprt{z}}$.
    \end{algorithmic}
\end{algorithm}
% \begin{algorithm}
%     \caption{Spectral Representation Learning for Instrument Variable Regression with Observable Confounding}
%     \label{alg:iv_confounding}
%     \begin{algorithmic}
%         \INPUT Function class $\mathcal{F}_{xo}$, $\mathcal{F}_{xzo}$, $\mathcal{F}_{yzo}$.
%         \STATE Estimate $\psi(z), \phi(x), V(o)$ with $\{x_i, z_i, o_i\}$ from $P(x, o|z)$ and obtained the estimated representation $\widehat{\phi}(x)$, $\widehat{W}_o$ and $\widehat{\psi}(z)$.
%         \STATE Estimate $\varphi(x)$, $\mu(o)$ with $\{x_i, o_i\}$ from $P(x|o)$ and obtain estimated $\widehat{\varphi(x)}$, $\widehat{\mu}(o)$.
%         \STATE Estimate $V(o)$, $\nu(y)$ with $P(y|z, o)$ and learned $\widehat{\psi}(z)$ and obtain estimated $\widehat{V}(o)$, $\widehat{\nu}(y)$.
%         \STATE Solve $\zeta(o)$, the feature space of $v(o)$ with the Equation \Tongzheng{TODO: a reference}
%         \STATE Solve the min-max optimization Equation~\eqref{eq:linear_primal_dual_iv_observable_confounder} and obtain the $V$ and $w$.
%         \OUTPUT $f(x, o) = \widehat{\phi}(x)^\top V \widehat{\zeta(o)}$ and $u(z, o) = w^\top W(o) \psi(z)$.
%     \end{algorithmic}
% \end{algorithm}

\begin{algorithm}
    \caption{Spectral Representation Learning for Instrument Variable Regression with Observable Confounding}
    \label{alg:iv_confounding}
    \begin{algorithmic}
        \INPUT Function classes $\calF_{xzo}$, $\calF_{yzo}$.
        \STATE Estimate $\psi: \calZ \to \bbR^{d_z}, \phi: \calX \to \bbR^{d_x}, V: \calO \to \bbR^{d_x \times d_z}$ with samples $\scbrk{x_i, z_i, o_i}$ to obtain the estimated representations $\widehat{\phi}$, $\widehat{V}$, and $\widehat{\psi}$.
        \STATE Estimate $Q: \calO \to \bbR^{d_x \times d_y}$, $\nu: \calY \to \bbR^{d_y}$ with samples $\scbrk{y_i, z_i, o_i}$ and learned representations $\widehat{\psi}$ and $\widehat{V}$ to obtain the estimated representations $\widehat{Q}$, $\widehat{\nu}$.
        % \STATE Solve $\zeta(o)$, the feature space of $v(o)$ with the Equation \Tongzheng{TODO: a reference}
        \STATE Solve the min-max optimization Equation~\eqref{eq:linear_primal_dual_iv_observable_confounder} and obtain the parameters $G$ and $w$.
        \OUTPUT $f: \sprt{x, o} \mapsto \inp{G, \widehat{V} \sprt{o}\transpose \otimes \widehat{\phi} \sprt{x}\transpose}$ and $g: \sprt{z, o} \mapsto \inp{w, \widehat{Q} \sprt{o}\transpose \widehat{V}\sprt{o} \widehat{\psi} \sprt{z}}$.
    \end{algorithmic}
\end{algorithm}

\paragraph{Implementation of Representation-Learning Stage.} For IV tasks, we parameterize $\phi: \calX \to \bbR^d$ and $\psi: \calZ \to \bbR^d$ using two neural networks. In the IV-OC setting, we extend this by additionally parameterizing $\nu: \calY \to \bbR^{d_y}$ with a neural network and introducing another network $\xi: \calO \to \bbR^{d_o}$ for the observable. To implement the low-rank factorization of $p_{X \given Z, O}$ and $p_{Y \given Z, O}$ for IV-OC, we introduce two learnable tensors, $P_V \in \bbR^{d_x \times d_z \times d_o}$ and $P_Q \in \bbR^{d_x \times d_y \times d_o}$. For any $o$, we then define $V \sprt{o} = P_V \otimes_3 \xi \sprt{o} \in \bbR^{d_x \times d_z}$, $Q \sprt{o} = P_Q \otimes_3 \xi \sprt{o} \in \bbR^{d_x \times d_y}$, and $W \sprt{o} = Q \sprt{o}\transpose V \sprt{o}$, which are subsequently incorporated into \eqref{eq:ivoc_factorization} and \eqref{eq:ivoc_factorization_y} for representation learning. A similar strategy is adopted for the PCL setting.
%%%%%%%%%%%%%%%%%%%%%%%%%%%%%%%%%%%%%%%%%%%%%%%%%%%%%%%%%%%%
\section{IDENTIFIABILITY AND ROLE OF THE REGULARIZER}
\label{sec:analysis}
%%%%%%%%%%%%%%%%%%%%%%%%%%%%%%%%%%%%%%%%%%%%%%%%%%%%%%%%%%%%

We now discuss how the choice of regularizer $\Omega$ influences the solution obtained by our method. While we focus on IV regression for clarity, similar considerations apply to the IV-OC and PCL settings.

Adding a regularizer $\Omega$ when the underlying integral equation has multiple solutions (\ie is not identifiable) is crucial. In traditional IV regression, identifiability is established \emph{before} considering estimation, typically through assumptions on the relationship between the instrument, treatment, and outcome (\eg, relevance and exclusion restrictions). These assumptions ensure a \emph{unique} solution to the integral equation. The estimation problem then focuses on finding this unique solution, often within a predefined function space \citep{liao2020provably}.

In contrast, our approach combines representation learning with estimation. We first learn a spectral representation, and \emph{then} solve a regularized optimization problem within the space spanned by the learned features. In general, \emph{without regularization}, the underlying integral equation will have infinitely many solutions. The regularizer $\Omega$ then acts as a \emph{selection mechanism}, choosing one particular solution from this infinite set and implicitly defining what we consider to be a ``good'' solution.

Our spectral representation, as described in Section~\ref{sec:method}, implicitly defines a norm depending on the underlying base measure. The eigenfunctions (and corresponding feature maps) obtained from the singular value decomposition of the conditional expectation operator are inherently tied to this base measure. In IV regression, we typically use the marginal distribution of the treatment, $\bbP_X$, and the euclidean norm on the space $L_2 \sprt{\bbP_X}$. Because $f$ can be represented as $f = \inp{\phi \sprt{\cdot}, u}$ (Proposition~\ref{prop:primal_iv}), an alternative norm can be defined directly on the representation space as $\norm{f}_{\calR \sprt{\Phi}} \triangleq \norm{u}_2$. This norm is also tied to $\bbP_X$ through the learned representation $\phi$.

However, other base measures are possible.  Consider, for example, an off-policy evaluation setting where we want to estimate the causal effect under a target policy $\pi$ different from the behavior policy that generated the data. In this case, it might be more appropriate to use, \eg, a base measure weighted by the importance sampling ratio $\frac{\diff \pi}{\diff \bbP_X}$. This different base measure would result in different eigenfunctions and a different notion of what constitutes a ``simple'' or ``smooth'' solution. This highlights a crucial point: the choice of base measure, often implicit, significantly influences the learned representation and, consequently, the solution selected by the regularizer when the underlying problem is under-identified.

\section{OMITTED PROOFS}
\label{app:method}
%%%%%%%%%%%%%%%%%%%%%%%%%%%%%%%%%%%%%%%%%%%%%%%%%%%%%%%%%%%%

We discuss the omitted proofs of Section~\ref{sec:method}.

%%%%%%%%------------------------------------------------
\subsection{Proof of Proposition~\ref{prop:linear-range-IV}}
%%%%%%%%------------------------------------------------

\linearrangeIV*

\begin{proof}
    Let $f \in L_2 \sprt{\bbP_x}$. By definition of the conditional expectation operator $E$, we have
    \begin{align*}
        E f &= \int_\calX f \sprt{x} \,\bbP_{X \given Z} \sprt{Z, \diff x} \\
        &= \int_\calX f \sprt{x} \frac{\diff \bbP_{X \given Z} \sprt{Z, \cdot}}{\diff \bbP_X} \sprt{x} \bbP_X \sprt{\diff x} \\
        &= \int_\calX f \sprt{x} \inp{\phi \sprt{x}, \psi \sprt{Z}} \bbP_X \sprt{\diff x}\,,
    \end{align*}
    where the last two equalities follow from Assumption~\ref{asp:low-rank-iv}. The conclusion follows from the linearity of $\inp{\cdot, \psi \sprt{Z}}$
    \begin{align*}
        E f &= \inp{\int_\calX \phi \sprt{x} f \sprt{x} \bbP_X \sprt{\diff x}, \psi \sprt{Z}} \\
        &= \inp{\psi \sprt{Z}, v_f}\,,
    \end{align*}
    where we defined $v_f = \int_\calX \phi \sprt{x} f \sprt{x} \bbP_X \sprt{\diff x}$.
\end{proof}

%%%%%%%%------------------------------------------------
\subsection{Proof of Proposition~\ref{prop:dual_iv}}
%%%%%%%%------------------------------------------------

\dualIV*

\begin{proof}
    For any $f \in L_2 \sprt{\bbP_X}$ and $g \in L_2 \sprt{\bbP_Z}$, the objective of the saddle-point problem~\eqref{eq:fredholm-minmax} considered for IV regression is defined as
    \begin{equation*}
        \calL \sprt{f, g} = \bbE \sbrk{g \sprt{Z} \sprt{Y - f \sprt{X}} - \frac12 g \sprt{Z}^2} + \lambda \Omega \sprt{f}\,.
    \end{equation*}
    We aim to find a reasonable function class $\calG$ that contains the maximizer $g_f^\star = \argmax_{g \in L_2 \sprt{\bbP_Z}} \calL \sprt{f, g}$ for any $f$. Note that by the tower rule, we can write
    \begin{align*}
        \calL \sprt{f, g} &= \bbE \sbrk{g \sprt{Z} \bbE \sbrk{Y - f \sprt{X} \given Z} - \frac12 g \sprt{Z}^2} + \lambda \Omega \sprt{f} \\
        &= \inp{g, r_0 - E f}_{L_2 \sprt{\bbP_Z}} - \frac12 \norm{g}_{L_2 \sprt{\bbP_Z}}^2 + \lambda \Omega \sprt{f}\,.
    \end{align*}
    $\calL \sprt{f, \cdot}$ is strongly concave. If a point $g_f^\star$ maximizes $\calL \sprt{f, \cdot}$ then it must be a zero of its Frechet derivative, which we denote $D_g \calL$,
    \begin{align*}
        D_g \calL \sprt{f, g_f^\star} = 0 &\quad\text{iff}\quad r_0 - E f - g_f^\star = 0 \\
        &\quad\text{iff}\quad g_f^\star = r_0 - E f\,.
    \end{align*}
    By Assumptions~\ref{asp:existence-sol-iv}, \ref{asp:low-rank-iv}, there exist vectors $v_0, v_f \in \bbR^d$ such that $r_0 = \inp{\psi \sprt{Z}, v_0}$ and $E f = \inp{\psi \sprt{Z}, v_f}$. This implies $g_f^\star$ takes the form $g_f^\star = \inp{\psi \sprt{Z}, v_0 - v_f}$, \ie $g_f^\star \in \calR \sprt{\Psi}$ and we can restrict the optimization problem to the class $\calG = \calR \sprt{\Psi}$.
\end{proof}

%%%%%%%%------------------------------------------------
\subsection{Proof of Proposition~\ref{prop:primal_iv}}
%%%%%%%%------------------------------------------------

\primalIV*

\begin{proof}
    Notice that the subspace $\calR \sprt{\Phi}$ of $L_2 \sprt{\bbP_X}$ is closed by continuity of the inner product in $\bbR^d$. Therefore, it is in direct sum with its orthogonal, \ie $L_2 \sprt{\bbP_X} = \calR \sprt{\Phi} \oplus \calR \sprt{\Phi}^\perp$ (Theorem 3.3.4. in \citealp{kreyszig1991introductory}). Following this observation, we write $f \in L_2 \sprt{\bbP_X}$ as $f = \inp{u_f, \phi \sprt{\cdot}} + h^\perp$, with $u_f \in \bbR^d$, $h^\perp \in \calR \sprt{\Phi}^\perp$, and we have
    \begin{align*}
        E f &= \int_\calX \sprt{\inp{u_f, \phi \sprt{x}} + h^\perp \sprt{x}} \inp{\phi \sprt{x}, \psi \sprt{Z}} \bbP_X \sprt{\diff x} & \text{(Assumption~\ref{asp:low-rank-iv})} \\
        &= \inp{u_f, \sprt{\int_\calX \phi \sprt{x} \phi \sprt{x}\transpose \bbP_X \sprt{\diff x}} \psi \sprt{Z}}_{\bbR^d} + \inp{h^\perp, \Phi \psi \sprt{Z}}_{L_2 \sprt{\bbP_X}} \\
        &= \inp{u_f, \sprt{\int_\calX \phi \sprt{x} \phi \sprt{x}\transpose \bbP_X \sprt{\diff x}} \psi \sprt{Z}}_{\bbR^d} & \sprt{h^\perp \in \calR \sprt{\Phi}^\perp}\,.
    \end{align*}
    Notice we also have
    \begin{align*}
        \norm{f}_{L_2 \sprt{\bbP_X}}^2 &= \bbE \sbrk{f \sprt{X}^2} \\
        &= \bbE \sbrk{\inp{u_f, \phi \sprt{X}}^2} + \bbE \sbrk{h^\perp \sprt{X}^2} & \sprt{h^\perp \perp \inp{u_f, \phi \sprt{\cdot}}} \\
        &= \inp{u_f, \bbE \sbrk{\phi \sprt{X} \phi \sprt{X}\transpose} u_f} + \bbE \sbrk{h^\perp \sprt{X}^2}\,.
    \end{align*}
    This shows $h^\perp$ influences the objective function only through $\bbE \sbrk{h^\perp \sprt{X}^2}$. Since we are minimizing for $f$, we can set $h^\perp = 0$ and only consider $f \in \calR \sprt{\Phi}$.
    
    Furthermore, Assumption~\ref{asp:existence-sol-iv} states there exists a solution to Problem~\ref{eq:fredholm_iv}. Using the same decomposition, this provides a solution to Problem~\ref{eq:fredholm_iv} in $\calR \sprt{\Phi}$. This concludes the proof.
\end{proof}

%%%%%%%%------------------------------------------------
\subsection{Proof of Proposition~\ref{prop:linear-range-IVOC}}
%%%%%%%%------------------------------------------------

\linearrangeIVOC*

\begin{proof}
    We follow the same reasoning as in the proof of Proposition~\ref{prop:linear-range-IV}. Let $f \in L_2 \sprt{\bbP_{X O}}$, and note we have
    \begin{align*}
        E f &= \int_\calX f \sprt{x, O} \,\bbP_{X \given Z, O} \sprt{\sprt{Z, O}, \diff x} \\
        &= \int_\calX f \sprt{x, O} \frac{\diff \bbP_{X \given Z, O} \sprt{\sprt{Z, O}, \cdot}}{\diff \bbP_X} \sprt{x} \bbP_X \sprt{\diff x} \\
        &= \int_\calX f \sprt{x, O} \inp{\phi \sprt{x}, V \sprt{O} \psi \sprt{Z}} \bbP_X \sprt{\diff x}\,,
    \end{align*}
    where the last two equalities follow from Assumption~\ref{asp:low-rank-ivoc}. The conclusion follows from the linearity of the inner product
    \begin{align*}
        E f &= \inp{\psi \sprt{Z}, V \sprt{O}\transpose \int_\calX \phi \sprt{x} f \sprt{x, O} \bbP_X \sprt{\diff x}} \\
        &= \inp{\psi \sprt{Z}, v_f \sprt{O}}\,,
    \end{align*}
    where we defined $v_f: o \mapsto V \sprt{o}\transpose \int_\calX \phi \sprt{x} f \sprt{x, o} \bbP_X \sprt{\diff x}$.
\end{proof}

%%%%%%%%------------------------------------------------
\subsection{Proof of Proposition~\ref{prop:primal-ivoc}}
%%%%%%%%------------------------------------------------

\primalIVOC*

\begin{proof}
    By plugging Equation~\eqref{eq:ivoc_factorization_y} into the RHS of Equation~\eqref{eq:fredholm_ivo}, we obtain
    \begin{equation*}
     \inp{\psi \sprt{z}, V \sprt{o}\transpose \sbrk{\int_\calX \phi \sprt{x} \phi\rbr{x}\transpose \bbP_X \sprt{\diff x}} v \sprt{o}} = \inp{\psi \sprt{z}, W \sprt{o}\transpose \int_\calY y \nu \sprt{y} \bbP_Y \sprt{\diff y}}\,,
    \end{equation*}
    Using the fact that $\psi$ spans every direction, we get
    \begin{equation*}
        V \sprt{o}\transpose \underbrace{\sbrk{\int \phi \sprt{x} \phi \sprt{x}\transpose \bbP_X \sprt{\diff x}}}_{\triangleq A} v \sprt{o} = W \sprt{o}\transpose \underbrace{\int y \nu \sprt{y} \bbP_Y \sprt{\diff y}}_{\triangleq \alpha}\,.
    \end{equation*}
    
    % \begin{align}
    %     \mathbb{E}_{y|z, o}[y] = \left\langle W(o)^\top \int y\nu(y) d P_y, \psi(z)\right\rangle.
    % \end{align}
    % As $\mathbb{E}_{x|z, o}[f(x, o)]=\mathbb{E}_{y|z, o}[y]$, this provides our target $o$'s representation of $u(z, o)$.

    We can now write $v$ as
    \begin{equation} \label{eq:vo-form}
        v \sprt{o} = A^\upplus \sprt{V \sprt{o}\transpose}^\upplus W \sprt{o}\transpose \alpha\,.
    \end{equation}
    Finding the space $v$ belongs to requires taking several (pseudo-)inverses, which is computationally expensive. We can instead use an alternative parametrization in $Q \sprt{o} \triangleq \sprt{V \sprt{o} \transpose}^\upplus W \sprt{o}\transpose$, which gives $W \sprt{o} = Q \sprt{o}\transpose V \sprt{o}$ in Equation~\eqref{eq:ivoc_factorization_y}. This leads to the desired parametrization.
\end{proof}

%%%%%%%%------------------------------------------------
\subsection{Proof of Proposition~\ref{prop:dual-ivoc}}
%%%%%%%%------------------------------------------------

\dualIVOC*

\begin{proof}
    The argument from Proposition~\ref{prop:dual_iv} also applies here.
\end{proof}

\section{EXPERIMENT DETAILS}
\label{app:setup}
%%%%%%%%%%%%%%%%%%%%%%%%%%%%%%%%%%%%%%%%%%%%%%%%%%%%%%%%%%%%

\paragraph{Setup.} For the IV tasks, each method is assessed using a test set comprising 2,000 samples, while for the PCL task, the evaluation is based on a test set of 500 samples. The test samples are randomly generated using the same data generation setting as the training set.
All experiments are conducted on a system equipped with an Intel(R) Xeon(R) Silver 4114 CPU @ 2.20GHz and a Quadro RTX 8000 GPU. 
Note that for IV with observable confounders, we simplify the implementation by omitting the decomposition of $p_{X \given Z}$. This helps reduce the number of parameters that need to be optimized.

\paragraph{Hyperparameters.} For the IV regression experiments, we employ the same hyper-parameter setting for DFIV, KIV, and DeepGMM used in \cite{xu2020dfiv}. For DE, we use the same Gaussian kernel with the bandwidth determined by the median trick \cite{singh2019kernel}. 
The network structures for \algabb on different datasets are provided in Table~\ref{tab: ours_dsprites} and \ref{tab: ours_demand}.

\paragraph{Dataset Dimensions.} The feature dimensions for these configurations are presented in Table~\ref{tab:datasets_dim}. 
\begin{table}[htb]
\caption{\small{Dimensions of treatment, instrument, and observable confounder Variables in two datasets with varying settings.
In the dSprites dataset, the 'low' and 'high' dimensions pertain to the instrument feature dimension.
In the Demand Design dataset, 'SC' represents separated observable confounders. \algabb and DFIV are assessed with separate representations for the observables. The remaining baseline methods are evaluated without separate observables, \ie, via incorporating these confounders into both treatment and instrumental variables.}}
\label{tab:datasets_dim}
\centering
\begin{tabular}{@{}lcccc@{}}
\toprule
\multirow{2}{*}{{\textbf{Dataset}}}  & \multicolumn{2}{c}{\textbf{dSprites}}   & \multicolumn{2}{c}{\textbf{Demand Design}}                            \\\cmidrule(lr){2-3} \cmidrule(lr){4-5}
 & Low-dim    & High-dim          & \multicolumn{1}{l}{With SC} & No SC      \\ \midrule
Treatment & 4,096& 4,096 & 784    & 1,569 \\
Instrument & 3  & 2,352 & 1      & 786  \\
Observable & -  & -    & 785    & -    \\\bottomrule                
\end{tabular}
\end{table}

\begin{table}[htb]
\centering
\caption{Network Structure of the feature extractor for Demand Design dataset. Each 2D convolutional layer (Conv2d) is depicted with \texttt{(input dimension, output dimension, kernel size, stride, padding)}.}
\label{tab: ours_imgnet}
\captionsetup{labelformat=empty}
\begin{minipage}{\textwidth}
\centering
\selectfont\setlength{\tabcolsep}{0.6em}
\begin{tabular}{cc}
\multicolumn{2}{c}{\textbf{ImageFeature}}\\
\hline
Layer & Configuration \\ \hline
1     & Input: 784      \\
2     & Conv2d(1, 16, 5, 1, 2), ReLU, MaxPool2d\\
3     & Conv2d(16, 32, 5, 1, 2), ReLU, MaxPool2d\\
4     & FC$(1568, 1024)$\\
\hline
\end{tabular}
\end{minipage}%
\end{table}

\begin{table}[htb]
\centering
\caption{Network Structure of \algabb for IV on the dSprites dataset. The tuple for each component represents the input/output dimensions within a particular component. FC refers to the fully-connected layers. BN represents batch normalization.}
\label{tab: ours_dsprites}
\captionsetup{labelformat=empty}
\begin{minipage}{.5\textwidth}
\centering
% \selectfont\setlength{\tabcolsep}{0.4em}
\begin{tabular}{cc}
\multicolumn{2}{c}{\textbf{Treatment Feature Net (32)}}\\
\hline
Layer & Configuration \\ \hline
1     & Input: 4096    \\
2     & FC$(4096, 1024)$, ReLU, BN      \\
3     & FC$(1024, 512)$, ReLU, BN     \\
4     & FC$(512, 128)$, ReLU, BN      \\ 
5     & FC$(128, 32)$, tanh      \\ 
\hline
\end{tabular}
\end{minipage}%
\begin{minipage}{.5\textwidth}
\centering
% \selectfont\setlength{\tabcolsep}{0.4em}
\begin{tabular}{cc}
\multicolumn{2}{c}{\textbf{Instrument Feature Net (32)}}\\
\hline
Layer & Configuration \\ \hline
1     & Input: 3    \\
2     & FC$(3, 256)$, ReLU, BN      \\
3     & FC$(256, 128)$, ReLU, BN     \\
4     & FC$(128, 128)$, ReLU, BN      \\ 
5     & FC$(128, 32)$, ReLU      \\ 
\hline
\end{tabular}
\end{minipage}
\rule{0pt}{.5cm} % Adjust the 2cm to your desired height
\begin{minipage}{.5\textwidth}
\centering
% \selectfont\setlength{\tabcolsep}{0.4em}
\begin{tabular}{cc}
\multicolumn{2}{c}{\textbf{Treatment Feature Net (64)}}\\
\hline
Layer & Configuration \\ \hline
1     & Input: 4096      \\
2     & FC$(4096, 1024)$, ReLU, BN      \\
3     & FC$(1024, 256)$, ReLU, BN     \\
4     & FC$(256, 64)$, tanh      \\ 
\hline
\end{tabular}
\end{minipage}%
\begin{minipage}{.5\textwidth}
\centering
% \fontsize{8.5}{10.5}\selectfont\setlength{\tabcolsep}{0.4em}
\begin{tabular}{cc}
\multicolumn{2}{c}{\textbf{Instrument Feature Net (64)}}\\
\hline
Layer & Configuration \\ \hline
1     & Input: 2352    \\
2     & FC$(2352, 1024)$, ReLU, BN      \\
3     & FC$(1024, 256)$, ReLU, BN     \\
4     & FC$(256, 64)$, ReLU\\
\hline
\end{tabular}
\end{minipage}%
\end{table}

\begin{table}[htb]
\centering
\caption{Network Structure of \algabb for IV on the Demand Design dataset. The tuple for each component represents the input/output dimensions within a particular component. FC refers to the fully-connected layers. BN represents batch normalization. ImageFeature indicates the feature extractor detailed in Table~\ref{tab: ours_imgnet}.}
\label{tab: ours_demand}
\captionsetup{labelformat=empty}
\begin{minipage}{.5\textwidth}
\centering
% \selectfont\setlength{\tabcolsep}{0.4em}
\begin{tabular}{cc}
\multicolumn{2}{c}{\textbf{Treatment Feature Net (32)}}\\
\hline
Layer & Configuration \\ \hline
1     & Input: 784 (P)      \\
2     & \text{ImageFeature}(P)\\
3     & FC$(1024, 512)$, ReLU, BN      \\
4     & FC$(512, 256)$, ReLU, BN     \\
5     & FC$(256, 32)$, tanh      \\ 
\hline
\end{tabular}
\vspace{6mm} % Adjusted space below this minipage
\end{minipage}%
\begin{minipage}{.5\textwidth}
\centering
% \selectfont\setlength{\tabcolsep}{0.4em}
\begin{tabular}{cc}
\multicolumn{2}{c}{\textbf{Observable Feature Net (32)}}\\
\hline
Layer & Configuration \\ \hline
1     & Input: 785 (T, S)      \\
2     & \text{ImageFeature}(S), T\\
3     & FC$(1025, 512)$, ReLU, BN      \\
4     & FC$(512, 256)$, ReLU,BN     \\
5     & FC$(256, 32)$, tanh      \\ 
\hline
\end{tabular}
\vspace{6mm} % Adjusted space below this minipage
\end{minipage}
\rule{0pt}{.5cm} % Adjust the 2cm to your desired height
\begin{minipage}{.45\textwidth}
\centering
% \selectfont\setlength{\tabcolsep}{0.4em}
\begin{tabular}{cc}
\multicolumn{2}{c}{\textbf{Instrument Feature Net (32)}}\\
\hline
Layer & Configuration \\ \hline
1     & Input: 1      \\
2     & FC$(1, 16)$, ReLU, BN      \\
3      & FC$(16, 4)$, ReLU      \\
\hline
\end{tabular}
\end{minipage}
\begin{minipage}{.45\textwidth}
\centering
% \selectfont\setlength{\tabcolsep}{0.4em}
\begin{tabular}{cc}
\multicolumn{2}{c}{\textbf{Outcome Feature Net (32)}}\\
\hline
Layer & Configuration \\ \hline
1     & Input: 1     \\
2     & FC$(1, 16)$, ReLU, BN      \\
3      & FC$(16, 4)$, ReLU      \\
\hline
\end{tabular}
\end{minipage}

\end{table}

\begin{table}[htb!]
\centering
\caption{Network Structure of \algabbpcl for PCL on the dSprites dataset. The tuple for each component represents the input/output dimensions within a particular component. FC refers to the fully-connected layers. BN represents batch normalization. }
\label{tab: ours_pcl}
\captionsetup{labelformat=empty}
\begin{minipage}{.45\textwidth}
\centering
% \selectfont\setlength{\tabcolsep}{0.4em}
\begin{tabular}{cc}
\multicolumn{2}{c}{\textbf{Treatment Feature Net}}\\
\hline
Layer & Configuration \\ \hline
1     & Input: 4096      \\
2     & FC$(4096, 1024)$, ReLU, BN      \\
3     & FC$(1024, 512)$, ReLU, BN     \\
4     & FC$(512, 32)$, tanh      \\ 
\hline
\end{tabular}
\vspace{6mm} % Adjusted space below this minipage
\end{minipage}%
\begin{minipage}{.45\textwidth}
\centering
% \selectfont\setlength{\tabcolsep}{0.4em}
\begin{tabular}{cc}
\multicolumn{2}{c}{\textbf{Observable Feature Net}}\\
\hline
Layer & Configuration \\ \hline
1     & Input: 4096       \\
2     & FC$(4096, 1024)$, ReLU, BN      \\
3     & FC$(1024, 256)$, ReLU,BN     \\
4     & FC$(256, 32)$, tanh      \\ 
\hline
\end{tabular}
\vspace{6mm}
\end{minipage}
\begin{minipage}{.45\textwidth}
\centering
% \selectfont\setlength{\tabcolsep}{0.4em}
\begin{tabular}{cc}
\multicolumn{2}{c}{\textbf{Instrument Feature Net}}\\
\hline
Layer & Configuration \\ \hline
1     & Input: 3      \\
2     & FC$(3, 16)$, ReLU, BN      \\
3     & FC$(16, 2)$, ReLU     \\
\hline
\end{tabular}
\end{minipage}
\begin{minipage}{.45\textwidth}
\centering
% \selectfont\setlength{\tabcolsep}{0.4em}
\begin{tabular}{cc}
\multicolumn{2}{c}{\textbf{Outcome Feature Net}}\\
\hline
Layer & Configuration \\ \hline
1     & Input: 1     \\
2     & FC$(1, 16)$, ReLU, BN      \\
3     & FC$(16, 2)$, ReLU     \\
\hline
\end{tabular}
\end{minipage}
\end{table}
\vspace{6mm}

%%%%%%%%%%%%%%%%%%%%%%%%%%%%%%%%%%%%%%%%%%%%%%%%%%%%%%%%%%%%
\section{DIMENSION MAPPING}
\label{app:data_mapping}
%%%%%%%%%%%%%%%%%%%%%%%%%%%%%%%%%%%%%%%%%%%%%%%%%%%%%%%%%%%%

We utilize the mapping function in \cite{bennett2019deep} to build the high-dimensional scenarios. Given a low-dimensional input $X_\text{low}\in\RR$, we generate a corresponding high-dimensional variable $X_\text{high}\in\RR^{784}$ via:
\begin{equation} \label{eq:map_to_high_dim}
    X_\text{high} = \text{RandomImage}(\pi(X_\text{low}))
\end{equation}
where $\pi(x)=\text{round}(\min(\max(1.5x+5,0),9))$ transforms the input to an integer within the range 0 to 9; $\text{RandomImage}(d)$ selects a random MNIST image corresponding to the input digit $d$.
%%%%%%%%%%%%%%%%%%%%%%%%%%%%%%%%%%%%%%%%%%%%%%%%%%%%%%%%%%%%
\section{DATA GENERATION FOR DSPRITES}
\label{app:data_dsprites}
%%%%%%%%%%%%%%%%%%%%%%%%%%%%%%%%%%%%%%%%%%%%%%%%%%%%%%%%%%%%

\subsection{IV Task}

We generate data using the following relationships:
\begin{align*}
    f(X) &= \frac{\|AX\|^2_2-5000}{1000},\\
    Y&= f(X)+32(\texttt{posX}-0.5)+\epsilon,\,\,\epsilon\sim\Ncal(0,0.5).
\end{align*}
where each entry of $A\in \RR^{10\times4096}$ is drawn from $\text{Uniform}(0,1)$. We generate $A$ once initially and maintain it fixed throughout the experiment. 

The relationship above allows us to generate the treatment $X\in\RR^{4096}$ and the instrument $Z\in\RR^{3}$, denoted as the low-dimensional scenario. 
By applying equation \eqref{eq:map_to_high_dim} to each element of $Z$, we can transform these into a high-dimensional scenario, \ie, $Z_\text{high}\in\RR^{2352}$. 

\subsection{PCL Task}
For the PCL setting, we employ the same definition of $f(\cdot)$ and obtain $Y$ through:
\begin{align*}
    Y&= \frac{1}{12}(\texttt{posX}-0.5)f(X)+\epsilon,\,\,\epsilon\sim\Ncal(0,0.5).
\end{align*}
We conduct the structural function estimation experiment on dSprites with treatment in $\RR^{4096}$ and instrument in $\RR^3$.
Following \cite{xu2021dfpv}, we fix the \texttt{shape} as \texttt{heart} and use (\texttt{scale}, \texttt{rotation}, \texttt{posX}) as the treatment-inducing proxy. We then sample another image from the dSprites dataset with the same \texttt{posY} as the output-inducing proxy, \ie, $\text{Image(\texttt{scale}=0.8, \texttt{rotation}=0, \texttt{posX}=0.5, \texttt{posY})}+{\eta}$, with $\eta\sim\Ncal(0,0.1I)$.
%%%%%%%%%%%%%%%%%%%%%%%%%%%%%%%%%%%%%%%%%%%%%%%%%%%%%%%%%%%%
\section{DATA GENERATION FOR DEMAND DESIGN}
\label{app:data_demand}
%%%%%%%%%%%%%%%%%%%%%%%%%%%%%%%%%%%%%%%%%%%%%%%%%%%%%%%%%%%%

Following \cite{singh2019kernel}, we generate data samples using the following relationships:
\begin{align*}
    Y&=f(P,T,S)+\epsilon\\
    f(P,T,S)&=100+(10+p)Sh(T)-2P\\
    h(t)&=2\rbr{\dfrac{(t-5)^4}{600}+\exp(-4(t-5)^2)+\dfrac{t}{10}-2}\\
\end{align*}
with 
\begin{align*}
        S&\sim\text{Uniform}\{1,\dots,7\}\\
        T&\sim\text{Uniform}[0,10]\\
        C&\sim\Ncal(0,1)\\
        V&\sim\Ncal(0,1)\\
        \epsilon&\sim\Ncal(\rho V, 1-\rho^2)\\
        P&=25+(C+3)h(T)+V
\end{align*}
The mapping function \eqref{eq:map_to_high_dim} is then used to generate high-dimensional versions of the variables $P$ and $S$. This results in $P, S\in \RR^{784}$ and $T,C\in\RR$.
%%%%%%%%%%%%%%%%%%%%%%%%%%%%%%%%%%%%%%%%%%%%%%%%%%%%%%%%%%%%
\section{ADDITIONAL EXPERIMENTS}
%%%%%%%%%%%%%%%%%%%%%%%%%%%%%%%%%%%%%%%%%%%%%%%%%%%%%%%%%%%%
\paragraph{Unlabeled Data Augmentation.}
\begin{figure}[!htp]
    \centering
    \includegraphics[width=0.7\linewidth]{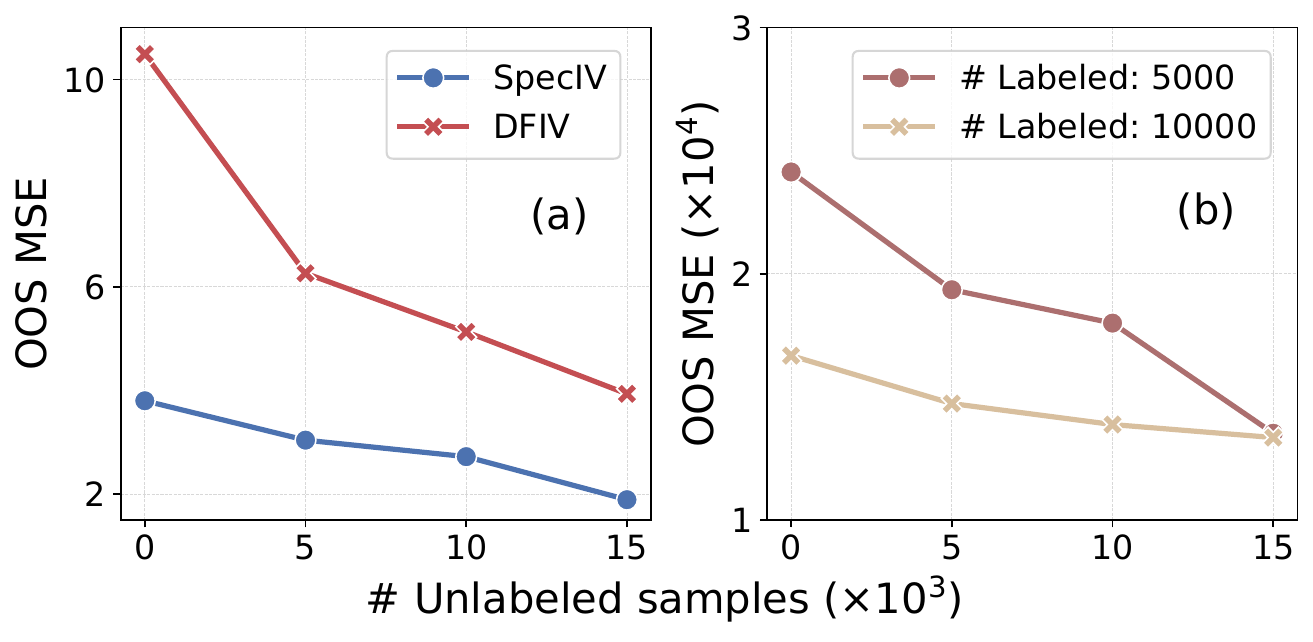}
    \caption{MSE comparison of (a) \algabb and DFIV on dSprites dataset with 5,000 labeled training samples and (b) \algabb on Demand Design Dataset with 5,000 and 10,000 labeled samples, respectively. Each method is trained with the specified amount of labeled data and an additional set of unlabeled samples.}
    \label{fig:add_unlabeled_data}
\end{figure}
In practical scenarios, it's often easier to collect a large amount of unlabeled data than labeled data. For instance, in the Demand Design task, the true ticket demand, which is the ground-truth label, is more challenging to acquire than other readily available information like ticket price and year time.
Therefore, it's crucial to explore how much an IV method can improve by using extra unlabeled data during its training process. Note that for \algabb with observable confounders, the unlabeled samples are used to optimize the decomposition in \eqref{eq:ivoc_factorization} while the original labeled data are employed for both \eqref{eq:ivoc_factorization} and \eqref{eq:ivoc_factorization_y}.
Figure~\ref{fig:add_unlabeled_data} shows that as the number of unlabeled training samples increases (up to 3x labeled data), the MSE for both methods decreases. This trend confirms that adding unlabeled data contributes positively to the model's performance. Notably, Figure~\ref{fig:add_unlabeled_data}a shows that \algabb consistently achieves a lower error rate compared to DFIV. Figure~\ref{fig:add_unlabeled_data}b illustrates that incorporating additional unlabeled data can effectively compensate for the lack of labeled data and enhance model performance.
This demonstrates a promising characteristic for practical applications where unlabeled data is more accessible.

\paragraph{Comparison with DeepIV~\citep{hartford2017deepiv}.}
We compare \algabb with DeepIV~\citep{hartford2017deepiv}, another important baseline, and observe that it exhibits inferior performance for both low- and high-dimensional treatment variables. Specifically, for uni-dimensional treatments, DeepIV is consistently outperformed by DFIV, as reported in Figure 4, Section 4.2, and Appendix E.2 by \cite{xu2020dfiv}. When the treatment variable is high-dimensional, DeepIV struggles to provide meaningful predictions due to the inherent challenges of performing conditional density estimation in high-dimensional spaces. 
To illustrate this, we report the MSE results on the Demand Design dataset, which consists of 10,000 training samples with a treatment dimension of 1,569. Table~\ref{tab:comp_deepiv} validate that with a high-dimensional treatment variable, DeepIV performs significantly worse than \algabb and other baselines.

\begin{table}[h]
    \centering
    \begin{tabular}{lccccc}
        \toprule
        Method & SpecIV & DFIV & DeepIV & KIV & DE \\
        \midrule
        OOS MSE ($\times 10^4$) & 1.67 & 2.40 & 4.31 & 2.85 & 2.89 \\
        \bottomrule
    \end{tabular}
    \caption{MSE comparison across different methods on Demand Design dataset with 10,000 training samples.}
    \label{tab:comp_deepiv}
\end{table}

\end{document}